%% file: main.tex
\documentclass[twoside,11pt]{article}

\usepackage[nohyperref]{jmlr2e}

\usepackage{lastpage}
\jmlrheading{26}{2025}{1-\pageref{LastPage}}{5/24}{1/25}{24-0737}{Rahul Parhi, Pakshal Bohra, Ayoub El Biari, Mehrsa Pourya, and Michael Unser}

\ShortHeadings{Random ReLU Neural Networks as Non-Gaussian Processes}{Parhi, Bohra, {El Biari}, Pourya, and Unser}
\firstpageno{1}

\usepackage{macros}
\allowdisplaybreaks[4]

\usepackage{hyperref}
\input{cleveref_setup}

\hypersetup{hidelinks}

\begin{document}

\title{Random ReLU Neural Networks as Non-Gaussian Processes}

\author{%
\name \name Rahul Parhi \email rahul@ucsd.edu \\
       \addr Department of Electrical and Computer Engineering \\
        University of California, San Diego \\
       La Jolla, CA 92093, USA
       \AND
\name{Pakshal Bohra} \email{pakshalbohra@gmail.com}\\
\name{Ayoub {El Biari}} \email{ayoubelbiari@gmail.com} \\
\name{Mehrsa Pourya} \email{mehrsa.pourya@epfl.ch} \\
\name{Michael Unser} \email{michael.unser@epfl.ch} \\
\addr Biomedical Imaging Group \\
\'Ecole polytechnique f\'ed\'erale de Lausanne \\
CH-1015 Lausanne, Switzerland}

\editor{Mohammad Emtiyaz Khan}

\maketitle

\begin{abstract}%
\input{sections/abstract}

\end{abstract}


\input{sections/introduction}

\input{sections/generalized-random-functions}

\input{sections/Radon}

\input{sections/well-defined}

\input{sections/properties}

\input{sections/asymptotic}
\input{sections/conclusion}

\acks{%
The authors would like to thank the anonymous reviewers and the action editor for their careful reading of the manuscript.
This work was supported in part by the Swiss National Science Foundation under Grant 200020\_219356 / 1 and in part by the European Research Council (ERC Project FunLearn) under Grant 101020573.
}

\appendix

\input{appendix/cf-sde}
\input{appendix/properties}

\input{appendix/asymptotic}
\input{appendix/more-figs}

\bibstyle{plain}
\bibliography{ref}

\end{document}

%% file: cleveref_setup.tex
\usepackage[capitalise,nameinlink]{cleveref}

\hypersetup{bookmarksopen=true}

\usepackage{crossreftools}
\crefformat{equation}{\textup{#2(#1)#3}}
\crefrangeformat{equation}{\textup{#3(#1)#4--#5(#2)#6}}
\crefmultiformat{equation}{\textup{#2(#1)#3}}{ and \textup{#2(#1)#3}}
{, \textup{#2(#1)#3}}{, and \textup{#2(#1)#3}}
\crefrangemultiformat{equation}{\textup{#3(#1)#4--#5(#2)#6}}%
{ and \textup{#3(#1)#4--#5(#2)#6}}{, \textup{#3(#1)#4--#5(#2)#6}}{, and \textup{#3(#1)#4--#5(#2)#6}}

\Crefformat{equation}{#2Equation~\textup{(#1)}#3}
\Crefrangeformat{equation}{Equations~\textup{#3(#1)#4--#5(#2)#6}}
\Crefmultiformat{equation}{Equations~\textup{#2(#1)#3}}{ and \textup{#2(#1)#3}}
{, \textup{#2(#1)#3}}{, and \textup{#2(#1)#3}}
\Crefrangemultiformat{equation}{Equations~\textup{#3(#1)#4--#5(#2)#6}}%
{ and \textup{#3(#1)#4--#5(#2)#6}}{, \textup{#3(#1)#4--#5(#2)#6}}{, and \textup{#3(#1)#4--#5(#2)#6}}

\crefname{remark}{Remark}{Remarks}
\crefname{lemma}{Lemma}{Lemmas}
\crefname{defn}{Definition}{Definitions}
\crefname{corollary}{Corollary}{Corollaries}
\crefname{prop}{Proposition}{Propositions}
\crefname{figure}{Figure}{Figures}

%% file: sections/abstract.tex
We consider a large class of shallow neural networks with randomly initialized parameters and rectified linear unit activation functions. We prove that these random neural networks are well-defined non-Gaussian processes. As a by-product, we demonstrate that these networks are solutions to stochastic differential equations driven by impulsive white noise (combinations of random Dirac measures). These processes are parameterized by the law of the weights and biases as well as the density of activation thresholds in each bounded region of the input domain. We prove that these processes are isotropic and wide-sense self-similar with Hurst exponent $3/2$. We also derive a remarkably simple closed-form expression for their autocovariance function. Our results are fundamentally different from prior work in that we consider a non-asymptotic viewpoint: The number of neurons in each bounded region of the input domain (i.e., the width) is itself a random variable with a Poisson law with mean proportional to the density parameter. Finally, we show that, under suitable hypotheses, as the expected width tends to infinity, these processes can converge in law not only to Gaussian processes, but also to non-Gaussian processes depending on the law of the weights. Our asymptotic results provide a new take on several classical results (wide networks converge to Gaussian processes) as well as some new ones (wide networks can converge to non-Gaussian processes).

%% file: sections/introduction.tex
\section{Introduction} \label{sec:intro}

A shallow (single-hidden-layer) neural network is a function of the form
\begin{equation}
    \vec{x} \mapsto \sum_{k=1}^N v_k \sigma(\vec{w}_k^\T\vec{x} - b_k), \quad \vec{x} \in \R^d,
    \label{eq:shallow-nn}
\end{equation}
where $\sigma:\R \rightarrow \R$ is the activation function, $N$ is the width of the network, and, for $k = 1, \ldots, N$, $v_k \in \R$ and $\vec{w}_k \in \R^d \setminus \curly{\vec{0}}$ are the weights and $b_k \in \R$ are the biases of the network. It is well-known that, as $N \to \infty$, several such networks with i.i.d. random weights and biases are equivalent to a Gaussian process~\citep{neal1996bayesian}. This result was extended to deep neural networks with i.i.d.\ random parameters by \citet{lee2018deep}. This correspondence enables exact Bayesian inference for regression using wide neural networks~\citep{williams1996computing, lee2018deep}.

Motivated by the tight link between wide neural networks and stochastic processes, we study properties of shallow rectified linear unit (ReLU) neural networks with randomly initialized parameters, henceforth referred to as \emph{random (ReLU) neural networks}. We study Poisson-type random functions of the form
\begin{equation}
    s_{\ReLU}(\vec{x}) = \sum_{k \in \Z} v_k \sq*{\ReLU(\vec{w}_k^\T\vec{x} - b_k) + \vec{c}_k^\T\vec{x} + c_{0, k}}, \quad \vec{x} \in \R^d,
    \label{eq:random-nn}
\end{equation}
where $\ReLU(t) \coloneqq  t_+ = \max\{0, t\}$, the $v_k$ are drawn i.i.d.\ with respect to the law $\Pr_V$ and the $(\vec{w}_k, b_k)$ are drawn such that
\begin{enumerate}
    \item the activation thresholds\footnote{The \emph{activation threshold} of the neuron $\vec{x} \mapsto \ReLU(\vec{w}^\T\vec{x} - b)$ is the hyperplane $H_{\vec{w}, b} = \curly{\vec{x} \in \R^d \st \vec{w}^\T\vec{x} = b}$.}
    are mutually independent;

    \item in expectation, the number of thresholds that intersect a finite volume in $\R^d$ is a constant (proportional to the product of a parameter $\lambda > 0$ and a property related to the geometry of the volume); and
    
    \item for every finite volume in $\R^d$, the thresholds are i.i.d.\ uniformly in the volume.
\end{enumerate}

The randomness that generates the $(\vec{w}_k, b_k)$ motivates the denomination \emph{Poisson} as it mimics the randomness in the jumps found in a unit interval of a compound Poisson process~\citep{daley2007introduction}. The parameter $\lambda > 0$ plays the role of the rate parameter of a compound Poisson process and controls the density of activation thresholds in each finite volume.
The correction terms $(\vec{x} \mapsto \vec{c}_k^\T\vec{x} + c_{0, k})_{k \in \Z}$ that appear in the sum are affine functions that ensure that the sum in \cref{eq:random-nn} converges almost surely. This is equivalent to imposing boundary conditions on $s_{\ReLU}$. These boundary conditions are crucial in proving that, under suitable hypotheses on $\Pr_V$, $s_{\ReLU}$ is a well-defined stochastic process. This is one of the primary technical contributions of this paper. Similar correction terms/boundary conditions appear in the definition of fractional Brownian motion~\citep{mandelbrot1968fractional} and L\'evy processes~\citep{ken1999levy,Jacob2001}.

By restricting our attention to compact subsets $\Omega \subset \R^d$, say, to the unit ball  $\B_1^d = \curly{\vec{x} \in \R^d \st \norm{\vec{x}}_2 \leq 1}$, we have that (see \cref{subsec:bounded-domain}) the process \cref{eq:random-nn} is realized by a random Poisson sum of the form
\begin{equation}
    \eval{s_{\ReLU}}_{\B_1^d} (\vec{x}) 
    = \vec{w}_0^\T\vec{x} + b_0 + \sum_{k = 1}^{N_\lambda} v_k \ReLU(\vec{w}_k^\T\vec{x} - b_k),
    \label{eq:random-nn-Poi-sum}
\end{equation}
where the width $N_\lambda$ is a Poisson random variable with mean $\lambda S$, where $S$ is proportional to the surface area of $\B_1^d$, and $\vec{w}_0^\T\vec{x} + b_0$ is an affine function. Thus, the form in the right-hand side of \cref{eq:random-nn-Poi-sum} is a \emph{finite-width} neural network with random parameters (including the width). The affine function $\vec{x} \mapsto \vec{w}_0^\T\vec{x} + b_0$ is a \emph{skip connection} in neural network parlance.
As $\lambda \to \infty$, we have that the expected value of the width satisfies $\E\sq{N_\lambda} \to \infty$.
Therefore, this limiting scenario corresponds to the asymptotic (i.e., infinite-width) regime.

\subsection{Contributions}
The purpose of this paper is to study the properties of random neural networks as in \cref{eq:random-nn,eq:random-nn-Poi-sum} for the class of admissible laws $\Pr_V$ (in the sense of \cref{defn:admissible}) which, for example, includes the Gaussian law. As these networks are completely specified by the law $\Pr_V$ and the rate parameter $\lambda > 0$, we let
\begin{equation}
    s_{\ReLU}(\dummy) \sim \ReLUP(\lambda; \Pr_V)
\end{equation}
denote that $s_{\ReLU}$ is generated according to the randomness described above, where $\ReLUP$ stands for \emph{ReLU process}. The main contributions of this paper are outlined below.

    \paragraph{Random ReLU Networks as Stochastic Processes} In \cref{sec:well-defined}, we prove that $s_{\ReLU}$ is a well-defined stochastic process. In doing so, we derive the so-called \emph{characteristic functional}\footnote{The characteristic functional of a stochastic process is analogous to the characteristic function of a random variable. See \cref{sec:gsp} for a detailed discussion.} of the process, which provides us with a complete characterization of its statistical distribution. Further, we show that $s_{\ReLU}$ is the unique continuous piecewise linear (CPwL) solution to the stochastic differential equation (SDE)
    \begin{equation}
        \TOp_{\ReLU} s \overset{\Law}{=} w \quad\subj\quad 
        \partial^\vec{m}s(\vec{0}) = 0, \abs{\vec{m}} \leq 1,
        \label{eq:SDE}
    \end{equation}
    where $\overset{\Law}{=}$ denotes equality in law and $\TOp_{\ReLU} = \KOp \RadonOp \Delta$ is the  \emph{whitening operator} for ReLU neurons. The driving term $w$ of the SDE is an \emph{impulsive white noise process} which is constructed from combinations of random Dirac measures.
    The boundary conditions $\partial^\vec{m}s(\vec{0}) = 0$, $\abs{\vec{m}} \leq 1$, are crucial in guaranteeing the existence of solutions to this SDE.
    In the form of the whitening operator, $\KOp$ is the filtering operator of computed tomography, $\RadonOp$ is the Radon transform, and $\Delta$ is the Laplacian (see \cref{sec:Radon} for a precise definition of these operators). The operator $\TOp_{\ReLU}$ was proposed by \citet{ongie2019function} to study the capacity of bounded-norm infinite-width ReLU networks. 

    \paragraph{Properties of Random ReLU Networks} In \cref{sec:properties}, we derive the first- and second-order statistics of $s_{\ReLU}$. Specifically, we present a remarkably simple closed-form expression for its autocovariance function. With the help of these statistics and the characteristic functional, we show that $s_{\ReLU}$ is a non-Gaussian process. We then show that $s_{\ReLU}$ is isotropic and wide-sense self-similar with Hurst exponent $H = 3/2$.
        
    \paragraph{Asymptotic Results} In \cref{sec:asymptotic}, we show that in the infinite-width regime ($\lambda \rightarrow \infty$), $s_{\ReLU}$ converges in law to a Gaussian process when $\Pr_V$ is a Gaussian law with a variance that is inversely proportional to $\lambda$. On the other hand, when $\Pr_V$ is a symmetric $\alpha$-stable (S$\alpha$S) law with $\alpha \in (1,2)$ and scaling parameter proportional to $\lambda^{-1/\alpha}$, $s_{\ReLU}$ converges in law to a non-Gaussian process.

\subsection{Related Work}
There is a large body of work that investigates the connections between neural networks with random initialization and stochastic processes. Early work in this direction is due to~\citet{neal1996bayesian} who proved that wide limits of shallow neural networks with bounded activation functions are Gaussian processes when the $(\vec{w}_k, b_k)$ are drawn i.i.d.\ with respect to any law and the $v_k$ are drawn i.i.d.\ with respect to a law that has zero mean and finite variance. More recently, it has been argued by many authors, with varying degrees of mathematical rigor, that deep neural networks with i.i.d.\ random initialization are Gaussian processes in wide limits~\citep{lee2018deep,matthews2018gaussian, garriga-alonso2018deep,novak2019bayesian,yang2019wide,dyer2019asymptotics,hanin2021random}.

Another line of work that is closely related to our setting is that of~\citet{yaida2020non}, who studies the stochastic processes realized by \emph{finite-width} random neural networks and shows that such processes are \emph{non-Gaussian}. The results of this paper are complementary to that of~\cite{yaida2020non} in that our finite-width networks as in \cref{eq:random-nn-Poi-sum} also correspond to non-Gaussian processes. However, our work is fundamentally different as we use the framework of generalized stochastic processes (see \cref{sec:gsp}). This allows us to derive the characteristic functional of the random neural network, which provides a complete description of its statistical distribution (i.e., the law of the process). The characteristic functional also allows us to easily study the limiting processes as the expected width $\E\sq{N_\lambda} \to \infty$, which \citet{yaida2020non} does not investigate.

In particular, we derive a novel and remarkably simple closed-form expression of the autocovariance function of the ReLU processes. Another important distinction of our asymptotic results compared to prior work on wide networks is that, in the asymptotic regime ($\lambda \to \infty$), the neural networks as in \cref{eq:random-nn,eq:random-nn-Poi-sum} can converge not only to Gaussian processes, but also to non-Gaussian processes, depending on the specific choice of $\Pr_V$. This type of result was alluded to by \citet{neal1996bayesian} in the case of S$\alpha$S initialization, although theoretical arguments were not carried out. Thus, this paper is the first, to the best of our knowledge, to carry out a rigorous investigation of the convergence of wide networks to non-Gaussian processes.

%% file: sections/generalized-random-functions.tex
\section{Generalized Stochastic Processes}\label{sec:gsp}
The mathematical framework used in this paper is based on the theory of \emph{generalized stochastic processes}~\citep{MR0065060, MR0068769, gelfand1964generalized,ito1984foundations} as opposed to the more common ``time-series'' approach to studying stochastic processes. In this section, we present the relevant background on generalized stochastic processes. We also refer the reader to the book of \cite{unser2014introduction} for further background.
While this theory relies on some rather heavy concepts from functional analysis, it allows for elegant arguments to investigate the properties of the stochastic processes realized by the random neural networks in \cref{eq:random-nn,eq:random-nn-Poi-sum}.

Throughout this paper, we fix a complete probability space $(\Omega, \F, \Pr)$. Before we introduce this theory, we first recall some results from classical probability theory. A real-valued random vector $\vec{X}$ is a measurable function from the probability space $(\Omega, \F, \Pr)$ to $(\R^d, \cB(\R^d))$, where $\cB(\R^d)$ denotes the Borel $\sigma$-algebra on $\R^d$. The $\emph{law}$ of $\vec{X}$ is the \emph{pushforward measure}
\begin{equation}
    \Pr_\vec{X}(A) \coloneqq (\vec{X}_\sharp\Pr)(A) \coloneqq \Pr\paren*{\vec{X}^{-1}(A)} = \Pr\paren*{\curly{\omega \in \Omega \st \vec{X}(\omega) \in A}} = \Pr\paren*{\vec{X} \in A},
\end{equation}
for all $A \in \cB(\R^d)$. Consequently, the \emph{characteristic function} of $\vec{X}$ is the (conjugate) Fourier transform of $\Pr_\vec{X}$, given by
\begin{equation}
    \hat{\Pr}_\vec{X}(\vec{\xi}) = \E\sq{\e^{\imag \vec{X}^\T\vec{\xi}}}, \quad \vec{\xi} \in \R^d,
\end{equation}
where $\imag^2 = -1$.

Generalized stochastic processes are random variables that take values in the (continuous) dual of a \emph{nuclear space}. In the remainder of this section, let $\Nuc$ denote a nuclear space and $\Nuc'$ denote its dual. If $u \in \Nuc'$ and $\varphi \in \Nuc$, we let $\ang{u, \varphi}_{\Nuc' \times \Nuc}$ denote the the \emph{duality pairing} of $u$ and $\varphi$ (i.e., the evaluation of $u$ at $\varphi$). A prototypical example of a nuclear space is the Schwartz space $\Sch(\R^d)$ of smooth and rapidly decreasing test functions. Its dual $\Sch'(\R^d)$ is the space of tempered generalized functions.\footnote{This space is often referred to as the space of tempered \emph{distributions}. We adopt the nomenclature of tempered \emph{generalized functions} in this paper so as to not cause confusion with probability distributions.} In order to discuss random variables that take values in the dual of a nuclear space, we must equip that space with a $\sigma$-algebra.
\begin{definition}
    The \emph{cylindrical $\sigma$-algebra on $\Nuc'$}, denoted by $\cB_c(\Nuc')$, is the $\sigma$-algebra generated by cylinders of the form $\curly{u \in \Nuc' \st (\ang{u, \varphi_1}_{\Nuc' \times \Nuc}, \ldots, \ang{u, \varphi_N}_{\Nuc' \times \Nuc}) \in A}$, where $N \in \N \setminus \curly{0}$, $\varphi_1, \ldots, \varphi_N \in \Nuc$, and $A \in \cB(\R^N)$.
\end{definition}

We remark that when $\Nuc$ is not only nuclear, but also Fréchet, such as $\Sch(\R^d)$, the cylindrical $\sigma$-algebra $\cB_c(\Nuc')$ coincides with the Borel $\sigma$-algebra $\cB(\Nuc')$~\citep[see][]{fernique1967processus,ito1984foundations}.
\begin{definition}
    \label{def:gen_sto_pro}
    A \emph{generalized stochastic process} is a measurable mapping 
    \begin{equation}
        s: (\Omega, \F, \Pr) \to (\Nuc', \cB_c(\Nuc')).
    \end{equation}
    The \emph{law} of $s$ is then the probability measure $\Pr_s \coloneqq s_\sharp \Pr$ which is defined on $\cB_c(\Nuc')$. The \emph{characteristic functional}\footnote{The characteristic functional of a generalized stochastic process was introduced by \citet{kolmogorov1935transformation}.} of $s$ is the (conjugate) Fourier transform of $\Pr_s$, given by
    \begin{equation}
        \hat{\Pr}_s(\varphi) = \E\sq{\e^{\imag\ang{s, \varphi}_{\Nuc' \times \Nuc}}}, \quad \varphi \in \Nuc.
    \end{equation}
\end{definition}

Observe that this definition recovers the classical characteristic function for random vectors that take values in $\R^d$. Indeed, $\R^d$ is a nuclear space whose dual is $\R^d$. Furthermore, for any $(\vec{x}, \vec{\xi}) \in \R^d \times \R^d$, we have that $\ang{\vec{x}, \vec{\xi}}_{\R^d \times \R^d} = \vec{x}^\T\vec{\xi}$. %
The characteristic functional of a generalized stochastic process contains all statistical information of the process in the same way that the characteristic function of a classical random variable contains all statistical information of that random variable. Analogous to the finite-dimensional case, the Bochner--Minlos theorem (see \citet{minlos1959generalized}) says that a functional $\hat{\Pr}: \Nuc \to \mathbb{C}$ is the characteristic functional of a generalized stochastic process if and only if $\hat{\Pr}$ is continuous, positive definite, and satisfies $\hat{\Pr}(0) = 1$.

The attractive feature of the framework of generalized stochastic processes is that it covers not only classical stochastic processes, but also processes that do not admit a pointwise interpretation such as white noise processes. For example, a generalized Gaussian process is defined as follows.
\begin{definition} \label[definition]{defn:Gauss-process}
A \emph{generalized stochastic process} $s$ that takes values in $\Nuc'$ is said to be \emph{Gaussian} if its characteristic functional is of the form
\begin{equation}
    \hat{\Pr}_s(\varphi) = \exp\left(\imag \mu_s(\varphi) - \frac{1}{2}\Sigma_s(\varphi, \varphi)\right),
\end{equation}
where $\varphi \in \Nuc$, $\mu_s: \Nuc \rightarrow \R$ is the mean functional of the process, given by
\begin{equation}
\mu_s(\varphi) = \E\sq{\ang{s, \varphi}_{\Nuc' \times \Nuc}},     
\end{equation}
and $\Sigma_s: \Nuc \times \Nuc \rightarrow \R$ is the covariance functional of the process, given by
\begin{equation}
    \Sigma_s(\varphi_1, \varphi_2) = \E\sq{\left(\ang{s, \varphi_1}_{\Nuc' \times \Nuc} - \mu_s(\varphi_1)\right) \left(\ang{s, \varphi_2}_{\Nuc' \times \Nuc} - \mu_s(\varphi_2)\right)}.
\end{equation}
\end{definition}

The above definition is backwards compatible with 
classical Gaussian processes that are space-indexed, as shown by \citet{duttweiler1973rkhs}, yet it also includes Gaussian white noise~\citep{hida1967analysis}.

With this machinery in hand, the primary technical contributions of this paper are (i) to prove that, for any $\lambda > 0$ and any admissible $\Pr_V$ (in the sense of \cref{defn:admissible}), the random neural network $s_{\ReLU} \sim \ReLUP(\lambda; \Pr_V)$ is a generalized stochastic process that takes values in $\Sch'(\R^d)$, and (ii) to provide an explicit form of its (non-Gaussian) characteristic functional (\cref{sec:well-defined}). With the help of the latter, we then derive various properties of the stochastic process in the non-asymptotic regime (\cref{sec:properties}) and also study its asymptotic ($\lambda \to \infty$) behavior (\cref{sec:asymptotic}) for various $\Pr_V$.

%% file: sections/Radon.tex
\section{The Radon Transform and Related Operators} \label{sec:Radon} 
Our characterization of random ReLU neural networks as stochastic processes hinges on the whitening operator that appears in the SDE \cref{eq:SDE}. This operator is based on the Radon transform. In this section we introduce the relevant background on the Radon transform and related operators. We refer the reader to the books of \citet{RammRadonBook} and \citet{HelgasonIntegralGeometry} for an in depth treatment of the Radon transform. The Radon transform of $\varphi \in L^1(\R^d)$ is given by
\begin{equation}
    \RadonOp\curly{\varphi}(\Sphvar, t)  = \int_{\Sphvar^\T\vec{x} = t} \varphi(\vec{x}) \dd\vec{x}, \quad (\Sphvar, t) \in \cyl,
    \label{eq:Radon}
\end{equation}
where $\mathrm{d}\vec{x}$ denotes the integration against the $(d-1)$-dimensional Lebesgue measure on the hyperplane $\curly{\vec{x} \in \R^d \st \Sphvar^\T\vec{x} = t}$ and $\Sph^{d-1} = \curly{\vec{x} \in \R^d \st \norm{\vec{x}}_2 = 1}$ denotes the unit sphere in $\R^d$. Observe that the Radon transform of $\varphi$ is a \emph{even} since $(\Sphvar, t)$ and $(-\Sphvar, -t)$ parametrize the same hyperplane. The adjoint operator, or dual Radon transform, applied to $\phi \in L^\infty(\cyl)$ is given by
\begin{equation}
    \RadonOp^*\curly{\phi}(\vec{x}) = \int_{\Sph^{d-1}} \phi(\Sphvar, \Sphvar^\T\vec{x}) \dd\Sphvar, \quad \vec{x} \in \R,
\end{equation}
where $\mathrm{d}\Sphvar$ denotes integration against the surface measure of $\Sph^{d-1}$.

Let $\Sch(\cyl)$ denote the Schwartz space of smooth and rapidly decreasing functions on $\cyl$. The range of the Radon transform on $\Sch(\R^d)$, defined by $\Sch_{\RadonOp} \coloneqq \RadonOp\paren*{\Sch(\R^d)}$, is a closed subspace of $\Sch(\cyl)$~\citep[p.~60]{HelgasonIntegralGeometry}. Therefore, since $\Sch(\cyl)$ is nuclear, $\Sch_{\RadonOp}$ is also nuclear. The next proposition summarizes the continuity and invertibility of the Radon transform.
\begin{proposition}[\citealt{LudwigRadon,GelfandIntegralGeometry,HelgasonIntegralGeometry}] \label[proposition]{prop:Radon}
    The operator $\RadonOp$ continuously maps $\Sch(\R^d)$ into $\Sch(\cyl)$.
    Moreover,
    \begin{equation}
        \RadonOp^* \KOp \RadonOp =  \frac{1}{2(2\pi)^{d-1}} (-\Delta)^\frac{d-1}{2} \RadonOp^*
        \RadonOp =  \frac{1}{2(2\pi)^{d-1}} \RadonOp^* \RadonOp (-\Delta)^{\frac{d-1}{2}} = \Id
    \end{equation}
    on $\Sch(\R^d)$. The underlying operators\footnote{Non-integer powers of $(-\Delta)$ and $(-\partial_t^2)$ are understood in the Fourier domain.} are the Laplacian $\Delta = \sum_{n=1}^d \partial_{x_n}^2$ and the filtering operator $\KOp = \frac{1}{2(2\pi)^{d-1}} (-\partial_t^2)^{\frac{d-1}{2}}$. Furthermore, $\RadonOp: \Sch(\R^d) \to \Sch_{\RadonOp}$ is a homeomorphism with inverse $\RadonOp^{-1} = \RadonOp^* \KOp: \Sch_{\RadonOp} \to \Sch(\R^d)$.
\end{proposition}

%% file: sections/well-defined.tex
\section{Random ReLU Neural Networks as Stochastic Processes} \label{sec:well-defined}
In this section, we will prove that, for any $\lambda > 0$ and admissible $\Pr_V$, the random neural network $s \sim \ReLUP(\lambda; \Pr_V)$ is a well-defined stochastic process and derive its characteristic functional on $\Sch(\R^d)$. The admissibility conditions in \cref{defn:admissible} are rather mild and most choices of $\Pr_V$ (e.g., Gaussian, S$\alpha$S for $1 < \alpha \leq 2$, uniform, etc.) satisfy these hypotheses.

\begin{definition} \label[definition]{defn:admissible}
    We say that the probability measure $\Pr_V$ is \emph{admissible} if
    \begin{enumerate}
        \item it is a Lévy measure, i.e., it satisfies $\Pr_V(\curly{0}) = 0$ and $\int_\R \min\curly{1, v^2} \dd\Pr_V(v) < \infty$, and \label{item:Levy-measure}
        
        \item it has a first absolute moment, i.e., if $V \sim \Pr_V$, then $\E\sq{\abs{V}} < \infty$. \label{item:FAM}
    \end{enumerate}
\end{definition}

Given a ReLU neuron $\vec{x} \mapsto \ReLU(\vec{w}^\T\vec{x} - b)$ with $\vec{w} \in \R^d \setminus \curly{\vec{0}}$ and $b \in \R$, we observe that, thanks to the homogeneity of the ReLU,
\begin{equation}
    \ReLU(\vec{w}^\T\vec{x} - b) = \norm{\vec{w}}_2 \ReLU(\tilde{\vec{w}}^\T\vec{x} - \tilde{b}),
\end{equation}
where $\tilde{\vec{w}} = \vec{w} / \norm{\vec{w}}_2$ and $\tilde{b} = b / \norm{\vec{w}}_2$. Therefore, the space of functions representable by shallow ReLU neural networks with input weights constrained to be unit norm is the same as the space of functions representable by shallow ReLU neural networks without constraints on the weights~\citep{parhi2023deep,shenouda2024variation}. To that end, we focus on neurons of the form $\ReLU(\vec{w}^\T\vec{x} - b)$ with $(\vec{w}, b) \in \cyl$.

An important property of the operator $\TOp_{\ReLU} = \KOp \RadonOp \Delta$ is that it ``whitens'' ReLU neurons. This result was implicitly proven by~\citet[Example~1]{ongie2019function}, explicitly proven by~\citet[Lemma~17]{parhi2021banach}, and then further investigated by, e.g., \citet[Lemma~5.6]{bartolucci2023understanding} and \citet[Corollary~11]{UnserRidges}. The whitening property is summarized in the following proposition.
\begin{proposition} \label[proposition]{prop:whiten}
    For any ReLU neuron
    \begin{equation}
        r_{(\vec{w}, b)}(\vec{x}) = \ReLU(\vec{w}^\T\vec{x} - b)
    \end{equation}
    with $(\vec{w}, b) \in \cyl$, we have that
    \begin{equation}
        \TOp_{\ReLU} r_{(\vec{w}, b)} = \deltae_{(\vec{w}, b)},
        \label{eq:whiten}
    \end{equation}
    where $\deltae_\vec{z} = (\delta_\vec{z} + \delta_{-\vec{z}}) / 2$ denotes the even symmetrization of the Dirac measure $\delta_\vec{z}$ supported at $\vec{z} \in \cyl$.
\end{proposition}

The equality in \cref{eq:whiten} is understood in $\Me(\cyl)$, the subspace of even finite (Radon) measures on $\cyl$. The arguments of the proof are based on duality. Indeed, observe that the adjoint of $\TOp_{\ReLU}$ takes the form $\TOp_{\ReLU}^* = \Delta \RadonOp^* \KOp$ (since $\Delta$ and $\KOp$ are self-adjoint). Furthermore, from \cref{prop:Radon} combined with the fact that $\Delta: \Sch(\R^d) \to \Sch(\R^d)$ is continuous, we see that $\TOp_{\ReLU}^*: \Sch_{\RadonOp} \to \Sch(\R^d)$ is continuous. Therefore, by duality, $\TOp_{\ReLU}: \Sch'(\R^d) \to \Sch_{\RadonOp}'$ is continuous. Since $r_{(\vec{w}, b)} \in \Sch'(\R^d)$, we have that $\TOp_{\ReLU} r_{(\vec{w}, b)}$ is indeed well-defined. Finally, $\Me(\cyl)$ is continuously embedded in $\Sch_{\RadonOp}'$ and so any finite measure in the range of $\KOp \RadonOp$ can be concretely identified to have even symmetries~\citep[see][for a detailed discussion]{UnserRidges,parhi2024distributional}. These symmetries are evidenced by the fact that the Radon transform of a ``classical'' function is necessarily even from the integral form in \cref{eq:Radon}.

\Cref{prop:whiten} motivates us to study Radon-domain impulsive white noises that are realized by Poisson-type random measures of the form
\begin{equation}
    w_\Poi = \sum_{k \in \Z} v_k \, \deltae_{(\vec{w}_k, b_k)},
    \label{eq:impulsive-Poi}
\end{equation}
where $v_k \overset{\text{i.i.d.}}{\sim} \Pr_V$ for some admissible $\Pr_V$ (in the sense of \cref{defn:admissible}) and the collection of random variables $\paren*{(\vec{w}_k, b_k)}_{k \in \Z}$ is a (homogeneous) Poisson point process\footnote{For a general treatment of point processes, we refer the reader to the book of \citet{daley2007introduction}.} on $\cyl$ with rate parameter $\lambda > 0$. This point process satisfies the following properties.
\begin{enumerate}
    \item The $(\vec{w}_k, b_k)$ are mutually independent.
    \item For any measurable subset $\Pi \subset \cyl$, if we define the random variable
    \begin{equation}
        N_\Pi = \abs{\curly{(\vec{w}_k, b_k) \st (\vec{w}_k, b_k) \in \Pi}},
    \end{equation}
    then
    \begin{equation}
        \Pr\paren{N_\Pi = n} = \frac{(\lambda \abs{\Pi})^n}{n!} \e^{-\lambda \abs{\Pi}},
    \end{equation}
    where $\abs{\Pi}$ denotes the $d$-dimensional Hausdorff measure of $\Pi$. That is to say, $N_\Pi$ is a Poisson random variable with mean $\lambda \abs{\Pi}$.
    
    \item For any measurable subset $B \subset \cyl$,
    \begin{equation}
        \Pr\paren*{(\vec{w}_k, b_k) \in B \given (\vec{w}_k, b_k) \in \Pi} = \frac{\abs{B \cap \Pi}}{\abs{\Pi}}. \label{eq:Poi-unif}
    \end{equation}
    That is to say, if a point lies in $\Pi$, then its location will be uniformly distributed on $\Pi$.
\end{enumerate}

Next, if we suppose that there exists a ``suitable'' right-inverse $\TOp_{\ReLU}^\dagger$ of $\TOp_{\ReLU}$ that satisfies  $\TOp_{\ReLU}\TOp_{\ReLU}^\dagger = \Id$  on  $\Sch_{\RadonOp}'$, then, intuitively, we could ``invert'' the result of \cref{prop:whiten} to find that $\TOp_{\ReLU}^\dagger\curly{w_\Poi}$ is precisely a random ReLU neural network generated in \cref{eq:random-nn}. It turns out that such a family of right-inverses exist. These inverses were first proposed by~\citet[Lemma~21]{parhi2021banach} in order to prove representer theorems for neural networks. Some further properties of these operator were identified by \citet{parhi2022kinds} and \citet{UnserRidges}. We summarize the properties from \citet[Lemma~21]{parhi2021banach} and \citet[Theorem~13]{UnserRidges} that are required for our investigation in the next proposition.

\begin{proposition}
\label[proposition]{prop:inverse}
    For any $\varepsilon > 0$, there exists an operator $\TOp_{\ReLU}^{\dagger\varepsilon}$ defined on $\Sch_{\RadonOp}'$ such that, for any $w \in \Sch_{\RadonOp}'$,
    \begin{align}
        &\TOp_{\ReLU}\TOp_{\ReLU}^{\dagger\varepsilon} w = w, \\
        &\sq*{(\partial^\vec{m} g_d^\varepsilon) *  \TOp_{\ReLU}^{\dagger\varepsilon}\curly{w}}(\vec{0}) = 0, \abs{\vec{m}} \leq 1, \label{eq:boundary}
    \end{align}
    where $g_d^\varepsilon: \R^d \to \R$ is the multivariate Gaussian probability density function with mean $\vec{0}$ and covariance matrix $\diag(\varepsilon, \ldots, \varepsilon)$. The restriction of $\TOp_{\ReLU}^{\dagger\varepsilon}$ to the subspace $\Me(\cyl) \subset \Sch_{\RadonOp}'$ continuously maps $\Me(\cyl)$ to $\Sch'(\R^d)$. This mapping is realized by the integral operator
    \begin{equation}
        \eval{\TOp_{\ReLU}^{\dagger\varepsilon}}_{\Me(\cyl)}\curly{\mu}(\vec{x}) = {\int_{\cyl} k^\varepsilon_\vec{x}(\Sphvar, t) \dd\mu(\Sphvar, t)} \label{eq:integral-measure}
    \end{equation}
    whose kernel is given by
    \begin{align*}
        k^\varepsilon_\vec{x}(\Sphvar, t)
        &= \ReLU(\Sphvar^\T\vec{x} - t) - \frac{(\Sphvar^\T\vec{x} - t)}{2} - \paren*{g_1^\varepsilon * \frac{\abs{\dummy}}{2}}(t) + (\Sphvar^\T\vec{x})\paren*{g_1^\varepsilon * \frac{\sgn}{2}}(t) \\
        &= \ReLU(\Sphvar^\T\vec{x} - t) + {\Sphvar_0^\varepsilon}^\T\vec{x} + t_0^\varepsilon, \numberthis \label{eq:kernel-op}
    \end{align*}
    where 
    $\sgn$ is the signum function. 
    Furthermore, there exists a universal constant $C > 0$ such that
    \begin{equation}
        \abs{k^\varepsilon_\vec{x}(\Sphvar, t)} \leq C (1 + \norm{\vec{x}}_2) \quad \text{for all } (\Sphvar, t) \in \cyl.
        \label{eq:unif-bound}
    \end{equation}
\end{proposition} 
\begin{remark}
    The purpose of introducing the $\varepsilon$-indexed right-inverse operators is for a mollification argument. We will eventually consider the limit $\varepsilon \to 0$ (see the proof of \cref{thm:well-defined} in \cref{app:well-defined}).
\end{remark}

With this inverse operator, we observe that, if $w_\Poi$ is an impulsive Poisson noise with rate $\lambda > 0$ and weights drawn i.i.d.\ according to $\Pr_V$ (as in \cref{eq:impulsive-Poi}), then, for any $\varepsilon > 0$,
\begin{align*}
    \TOp_{\ReLU}^{\dagger\varepsilon}\curly{w_\Poi}
    &= \TOp_{\ReLU}^{\dagger\varepsilon}\curly*{\sum_{k \in \Z} v_k \, \deltae_{(\vec{w}_k, b_k)}} \\
    &= \sum_{k \in \Z} v_k   \TOp_{\ReLU}^{\dagger\varepsilon}\curly*{\deltae_{(\vec{w}_k, b_k)}} \\
    &= \sum_{k \in \Z} v_k \sq*{\ReLU(\vec{w}_k^\T(\dummy) - b_k) + \vec{c}_k^{\varepsilon\T}(\dummy) + c^\varepsilon_{0, k}}, \numberthis \label{eq:Poi-to-ReLU}
\end{align*}
where the second line is justified due to the uniform bound in \cref{eq:unif-bound}, and the third line follows from \cref{eq:integral-measure}. Therefore, $\TOp_{\ReLU}^{\dagger\varepsilon}\curly{w_\Poi}$ is a random neural network as in \cref{eq:random-nn} that satisfies the boundary conditions in \cref{eq:boundary}. We write
\begin{equation}
    s_{\ReLU}^\varepsilon \sim \ReLUP^\varepsilon(\lambda; \Pr_V)
\end{equation}
to denote that $s_{\ReLU}^\varepsilon$ is such a random neural network. Furthermore, we let
\begin{equation}
   s_{\ReLU} \sim \ReLUP(\lambda; \Pr_V),
\end{equation}
as introduced in \cref{sec:intro}, correspond to a random ReLU neural network that satisfies the the limiting boundary conditions as $\varepsilon \to 0$. That is to say, $\partial^{\vec{m}} s_{\ReLU}(\vec{0}) = 0$, $\abs{\vec{m}} \leq 1$, with the convention that the value of a piecewise constant function at a jump is the middle value.

In the next theorem, we prove that these random neural networks are well-defined stochastic process that take values in $\Sch'(\R^d)$ and provide a complete statistical characterization through their characteristic functional.
\begin{theorem} \label{thm:well-defined}
    For any $\varepsilon > 0$, $\lambda > 0$, and admissible $\Pr_V$ (in the sense of \cref{defn:admissible}), the random neural network $s_{\ReLU}^\varepsilon \sim \ReLUP^\varepsilon(\lambda; \Pr_V)$ is a measurable mapping
    \begin{equation}
        s_{\ReLU}^\varepsilon: (\Omega, \F, \Pr) \to (\Sch'(\R^d), \cB(\Sch'(\R^d))
    \end{equation}
    with characteristic functional given by
    \begin{equation}\label{eq:cf_srelu_mol}
        \hat{\Pr}_{s_{\ReLU}^\varepsilon}(\varphi) = \exp\paren*{\lambda \int_{\R} \int_\R \int_{\Sph^{d-1}} \paren*{\e^{\imag v \TOp_{\ReLU}^{\dagger\varepsilon*}\curly{\varphi}(\Sphvar, t)} - 1} \dd\Sphvar \dd t \dd\Pr_V(v)}, \quad \varphi \in \Sch(\R^d),
    \end{equation}
     where $\mathrm{d}\Sphvar$ denotes integration against the surface measure on $\Sph^{d-1}$ and
    \begin{equation}
        \TOp_{\ReLU}^{\dagger\varepsilon*}: \varphi \mapsto \int_{\R^d} k^\varepsilon_\vec{x}(\dummy) \varphi(\vec{x}) \dd\vec{x} \label{eq:inverse-op-adjoint}
    \end{equation}
    is the adjoint\footnote{Observe that $\TOp_{\ReLU}^{\dagger\varepsilon*}$ is well-defined on $\Sch(\R^d)$ thanks to \cref{eq:unif-bound}.} of $\TOp_{\ReLU}^{\dagger\varepsilon}$.
    Furthermore, $s_{\ReLU}^\varepsilon$ is the unique CPwL solution to the SDE
    \begin{equation}
        \TOp_{\ReLU} s \overset{\Law}{=} w_\Poi \quad\subj\quad \sq*{(\partial^\vec{m} g_d^\varepsilon) *  s}(\vec{0}) = 0, \abs{\vec{m}} \leq 1, \label{eq:SDE-thm-mol}
    \end{equation}
    among all tempered weak solutions,\footnote{A tempered weak solution to the SDE is any random tempered generalized function $s^\star \in \Sch'(\R^d)$ that satisfies \cref{eq:SDE-thm-mol}. Such a solution is referred to as ``tempered'' as it lies in $\Sch'(\R^d)$ and ``weak'' since the action of $\TOp_{\ReLU}$ on $s^\star$ is understood by duality.} 
    where $w_\Poi$ is an impulsive Poisson noise with rate $\lambda$ and weights drawn i.i.d.\ according to $\Pr_V$ (as in \cref{eq:impulsive-Poi}). All other tempered weak solutions to the SDE take the form $s_{\ReLU}^\varepsilon + h$, where $h$ is a harmonic polynomial of degree $\geq 2$.\footnote{A harmonic polynomial $h$ is a polynomial defined on $\R^d$ such that $\Delta h = 0$ on all of $\R^d$.}

    Finally, in the limiting scenario ($\varepsilon \to 0$), we have that $s_{\ReLU} \sim \ReLUP(\lambda; \Pr_V)$ is a measurable mapping $(\Omega, \F, \Pr) \to (\Sch'(\R^d), \cB(\Sch'(\R^d))$
    with characteristic functional given by
    \begin{equation}\label{eq:cf_srelu}
        \hat{\Pr}_{s_{\ReLU}}(\varphi) = \exp\paren*{\lambda \int_{\R} \int_\R \int_{\Sph^{d-1}} \paren*{\e^{\imag v \TOp_{\ReLU}^{\dagger*}\curly{\varphi}(\Sphvar, t)} - 1} \dd\Sphvar \dd t \dd\Pr_V(v)}, \quad \varphi \in \Sch(\R^d),
    \end{equation}
    where $\TOp_{\ReLU}^{\dagger*}$ is the limiting operator as $\varepsilon \to 0$ whose kernel is $k_\vec{x} \coloneqq \lim_{\varepsilon \to 0} k_\vec{x}^\varepsilon$ (pointwise limit). This random neural network is the unique CPwL solution to the SDE
    \begin{equation}
        \TOp_{\ReLU} s \overset{\Law}{=} w_\Poi \quad\subj\quad \partial^\vec{m}s(\vec{0}) = 0, \abs{\vec{m}} \leq 1. \label{eq:SDE-thm-new}
    \end{equation}
\end{theorem}

While the proof of the theorem is rather technical, the main ingredients can be divided into two steps. The first is to prove that $w_\Poi$ is a well-defined stochastic process that takes values in $\Sch_{\RadonOp}'$. The second is to invoke the computation in \cref{eq:Poi-to-ReLU} which linearly and continuously transforms $w_\Poi$ into a random ReLU neural network. This transformation allows us to derive the characteristic functional of $s_{\ReLU}$ in terms of the characteristic functional of $w_\Poi$. The proof appears in \cref{app:well-defined}.

\subsection{Restrictions to Compact Domains} \label{subsec:bounded-domain}
Recall from \cref{eq:random-nn-Poi-sum} that, for any $\lambda > 0$ and admissible $\Pr_V$ (in the sense of \cref{defn:admissible}), the restriction of the random neural network $s_{\ReLU} \sim \ReLUP(\lambda; \Pr_V)$ to a compact domain, say, the unit ball $\B_1^d$ is a random Poisson sum of the form
\begin{equation} \label{eq:s_relu_restricted}
    \eval{s_{\ReLU}}_{\B_1^d} (\vec{x}) 
    = \vec{w}_0^\T\vec{x} + b_0 + \sum_{k = 1}^{N_\lambda} v_k \ReLU(\vec{w}_k^\T\vec{x} - b_k),
\end{equation}
where the width $N_\lambda$ is a Poisson random variable. The reader can quickly check that the activation thresholds that intersect $\B_1^d$ correspond to Poisson points that lie in $\Sph^{d-1} \times [-1, 1]$. 
Thus, the number of neurons $N_\lambda$ is a Poisson random variable with mean $\lambda \abs{\Sph^{d-1} \times [-1, 1]}$, which is $\lambda$ multiplied by twice the surface area of the $(d-1)$-sphere. For general compact domains $\Omega \subset \R^d$, following \citet[Section~IV]{parhi2022near}, we define
\begin{equation}
    \mathcal{Z}_\Omega \coloneqq \curly{(\vec{w}, b) \in \cyl \st \curly{\vec{x} \st \vec{w}^\T\vec{x} = b} \cap \Omega \neq \varnothing}.
    \label{eq:Z-Omega}
\end{equation}
Then, the restriction $\eval{s_{\ReLU}}_{\Omega}$ is a random neural network whose width $N_{\lambda, \Omega}$ is a Poisson random variable with mean $\lambda \abs{\mathcal{Z}_{\Omega}}$. As $\lambda \to \infty$, we see that $\E\sq{N_{\lambda, \Omega}} \to \infty$. Therefore, the asymptotic setting ($\lambda \to \infty$) corresponds to the infinite-width regime.

%% file: sections/properties.tex
\section{Properties of Random ReLU Neural Networks} \label{sec:properties}
The characteristic functional \cref{eq:cf_srelu} allows us to derive the first- and second-order statistics of $s_{\ReLU}$ as well as infer some of its other properties such as isotropy and wide-sense self-similarity. We summarize these properties in \cref{th:properties}.

\begin{theorem}\label{th:properties}
    For $\lambda > 0$ and admissible $\Pr_V$ (in the sense of \cref{defn:admissible}), let $s_{\ReLU} \sim \ReLUP(\lambda; \Pr_V)$. Then, the following statements hold.
    \begin{enumerate}
        \item The mean of $s_{\ReLU}$ is given by
        \begin{equation}\label{eq:mean}
            \E\sq{s_{\ReLU}(\vec{x})} = \lambda \E\sq{V} \int_{\R} \int_{\Sph^{d-1}} k_{\vec{x}}(\vec{u}, t) \dd\vec{u} \dd t,
        \end{equation}
        where $k_{\vec{x}}$ is defined in \cref{thm:well-defined}.

        \item If $\Pr_V$ has a finite second moment, then the autocovariance of $s_{\ReLU}$ is given by
        \begin{align}
            C_{s_{\ReLU}}(\vec{x}, \vec{y}) &= \E\sq{(s_{\ReLU}(\vec{x}) - \E\sq{s_{\ReLU}(\vec{x})}) (s_{\ReLU}(\vec{y}) - \E\sq{s_{\ReLU}(\vec{y})})} \nonumber \\
            &=  A \lambda \E\sq{V^2} \paren*{\norm{\vec{x} - \vec{y}}_2^3 - \norm{\vec{x}}_2^3 - \norm{\vec{y}}_2^3 + 3 \vec{x}^\T\vec{y}\paren*{\norm{\vec{x}}_2 + \norm{\vec{y}}_2}}, \label{eq:autocov}
        \end{align}
        where $A = \frac{\Gamma(-3/2)}{2^{d+3} \pi^{d/2} \Gamma((d+3)/2)}$ and $\Gamma(\dummy)$ is Euler's gamma function.

        \item The process $s_{\ReLU}$ is isotropic, i.e., it has the same probability law as its rotated version $s_{\ReLU}(\mat{U}^\T \cdot)$, where $\mat{U}$ is any $(d \times d)$ rotation matrix.

        \item If $\Pr_V$ has zero mean and a finite second moment, then $s_{\ReLU}$ is wide-sense self-similar with Hurst exponent $H=3/2$, i.e., it has the same second-order moments as its scaled and renormalized version $a^{H}s_{\ReLU}(\cdot/a)$ with $a>0$.

        \item The process $s_{\ReLU}$ is non-Gaussian.
    \end{enumerate}
\end{theorem}

The proof of \cref{th:properties} can be found in \cref{app:properties}. We mention that the expression of the autocovariance in \cref{eq:autocov} is remarkably simple. This is in contrast to prior works that either (i) do not provide a closed-form expression~\citep{lee2018deep,yaida2020non,hanin2021random}, or (ii) provide a closed-form expression, but do not consider the ReLU activation function~\citep{williams1996computing}. Furthermore, other than the work of \citet{yaida2020non}, these prior works only consider the infinite-width regime.

%% file: sections/asymptotic.tex
\section{Asymptotic Results} \label{sec:asymptotic}
In the literature, there has been a lot of work on studying the wide limits of random neural networks. Here, we present an asymptotic result for random ReLU neural networks with i.i.d. weights drawn from an S$\alpha$S law. The proof appears in \cref{app:asymptotic}.

\begin{theorem}\label{th:asymptotic}
    For $n \in \mathbb{N}$, let $s_{\ReLU}^{n} \sim \ReLUP(\lambda=n; \Pr_V)$ with $\Pr_V$ being a symmetric $\alpha$-stable law with scale parameter $b n^{(-1/\alpha)}$,\footnote{Similar to \citet{neal1996bayesian, lee2018deep}, the scale parameter inversely depends on the expected width of the network.} where $\alpha \in (1,2]$ and $b \in \R_{+}$, that is, $\hat{\Pr}_{V}(\xi) = \exp\left(-\frac{|b \xi|^{\alpha}}{n}\right)$. Then, we have 
    \begin{equation}
        s_{\ReLU}^{n} \xrightarrow[n \to \infty]{\Law} s_{\ReLU}^{\infty},
    \end{equation}
    where $s_{\ReLU}^{\infty}$ is a well-defined generalized stochastic process that takes values in $\Sch'(\R^d)$ and has the characteristic functional
    \begin{equation}
        \hat{\Pr}_{s_{\ReLU}^{\infty}}(\varphi) = \exp\left( - |b|^{\alpha} \|\TOp_{\ReLU}^{\dagger*}\curly{\varphi} \|_{L^{\alpha}}^{\alpha} \right), \quad \varphi \in \Sch(\R^d).
        \label{eq:wide-cf}
    \end{equation}
\end{theorem}

When $\alpha = 2$, the S$\alpha$S law is the Gaussian law. In this case, we can deduce that $s_{\ReLU}^{\infty}$ is indeed a Gaussian process (see \cref{app:asymptotic}). On the other hand, for $\alpha \in (1, 2)$, we can readily see that $s_{\ReLU}^{\infty}$ is non-Gaussian. Therefore, we have rigorously shown that wide limits of random neural networks are not necessarily Gaussian processes.

We illustrate these observations numerically in \cref{fig:Gauss-2D,fig:stable-2D}, where we generated random neural networks with $\Pr_V$ being Gaussian ($\alpha = 2$) and non-Gaussian ($\alpha=1.25$), respectively. There, we plot a top-down view of realizations of random neural networks for $\lambda \in \curly{1, 10, 100, 1000}$ where we color the linear regions with the magnitude of the gradient of the function. \Cref{fig:Gauss-2D}(d) looks like a two-dimensional Gaussian process, while \cref{fig:stable-2D}(d) remains to look CPwL (non-Gaussian). Discussion on how we generated the random neural networks numerically along with some additional figures appear in \cref{app:more-figs}.

\begin{figure}[t!]
    \centering
    \begin{minipage}[b]{0.24\linewidth}
        \centering
        \centerline{\includegraphics[width=\textwidth]{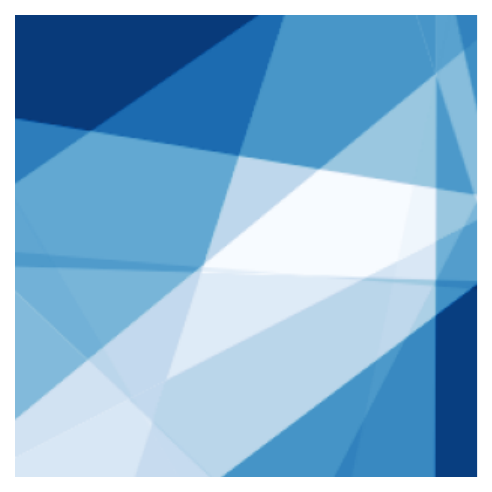}}
        (a) $\lambda = 1$
    \end{minipage}
    \begin{minipage}[b]{0.24\linewidth}
        \centering
        \centerline{\includegraphics[width=\textwidth]{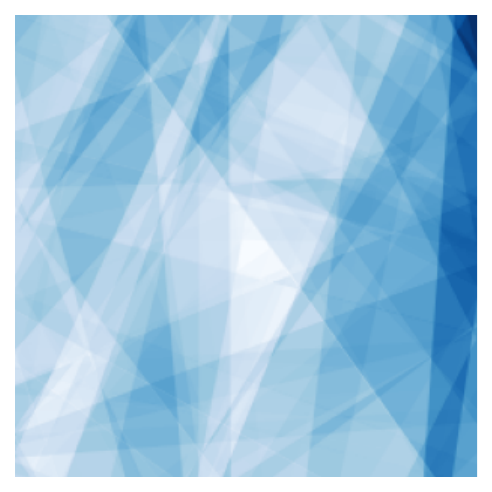}}
        (b) $\lambda = 10$
    \end{minipage}
    \begin{minipage}[b]{0.24\linewidth}
        \centering
        \centerline{\includegraphics[width=\textwidth]{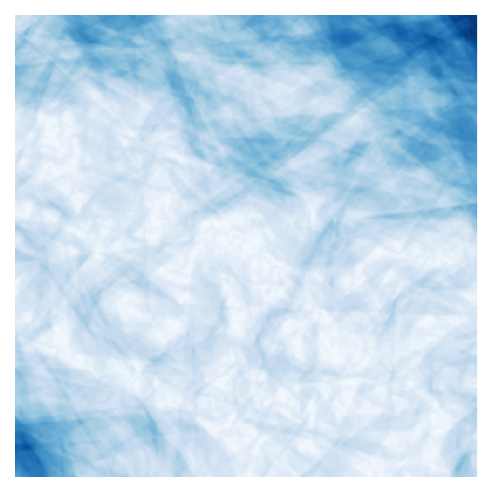}}
        (c) $\lambda = 100$
    \end{minipage}
    \begin{minipage}[b]{0.24\linewidth}
        \centering
        \centerline{\includegraphics[width=\textwidth]{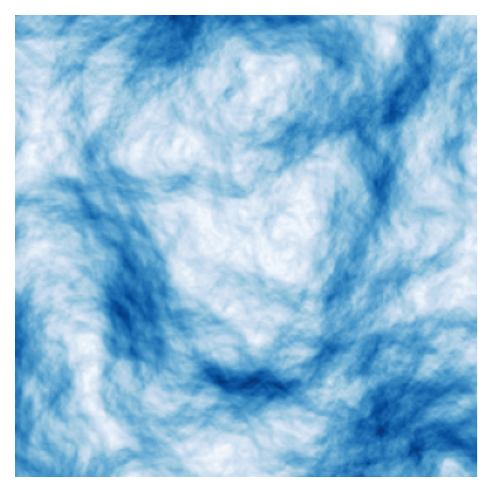}}
        (d) $\lambda = 1000$
    \end{minipage}
    \caption{$\Pr_V$ is Gaussian.}
    \label{fig:Gauss-2D}
\end{figure}
\begin{figure}[t!]
    \centering
    \begin{minipage}[b]{0.24\linewidth}
        \centering
        \centerline{\includegraphics[width=\textwidth]{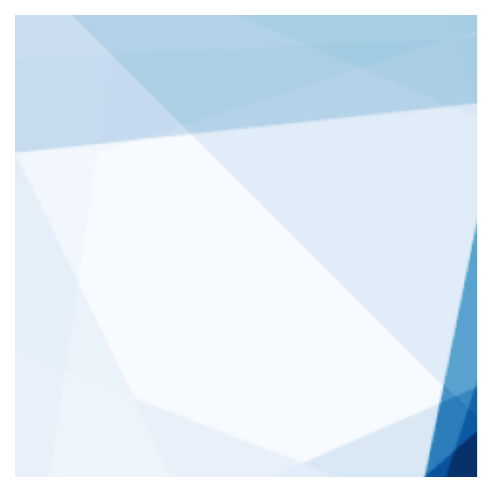}}
        (a) $\lambda = 1$
    \end{minipage}
    \begin{minipage}[b]{0.24\linewidth}
        \centering
        \centerline{\includegraphics[width=\textwidth]{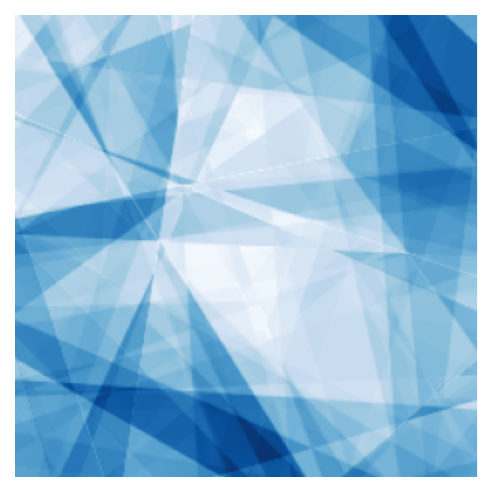}}
        (b) $\lambda = 10$
    \end{minipage}
    \begin{minipage}[b]{0.24\linewidth}
        \centering
        \centerline{\includegraphics[width=\textwidth]{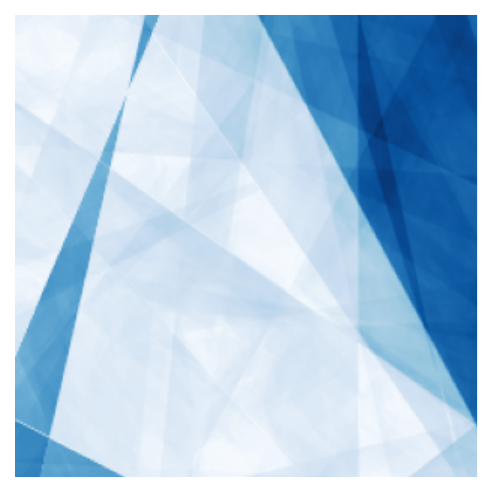}}
        (c) $\lambda = 100$
    \end{minipage}
    \begin{minipage}[b]{0.24\linewidth}
        \centering
        \centerline{\includegraphics[width=\textwidth]{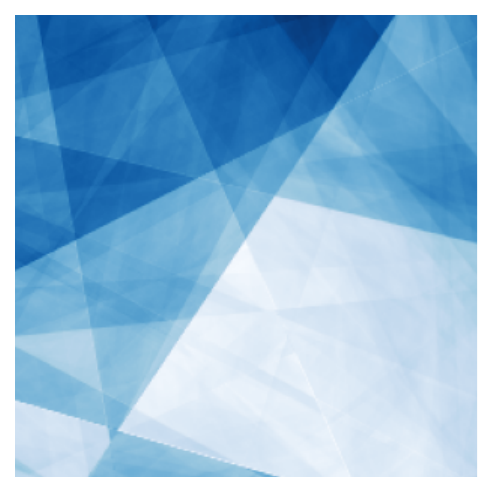}}
        (d) $\lambda = 1000$
    \end{minipage}
    \caption{$\Pr_V$ is symmetric ($\alpha = 1.25$)-stable.}
    \label{fig:stable-2D}
\end{figure}

%% file: sections/conclusion.tex
\section{Conclusion}

We have investigated the statistical properties of random ReLU neural networks. We proved that these networks are well-defined non-Gaussian processes in the non-asymptotic regime. We showed that these processes are isotropic and wide-sense self-similar with Hurst exponent $3/2$. Remarkably, the autocovariances of these processes have simple closed-form expressions. Finally, we showed that, under suitable hypotheses, as the expected width tends to infinity, these processes can converge in law not only to Gaussian processes, but also to non-Gaussian processes depending on the law of the weights. These asymptotic results recover the classical observation that wide networks converge to Gaussian processes as well as prove that wide networks can converge to non-Gaussian processes. Although the presented investigation only considered shallow random ReLU neural networks, an important direction of future work would be to generalize our exact characterizations to deeper networks. To that end, the techniques developed by \cite{zavatone2021exact} could provide a starting point for that investigation.

%% file: appendix/cf-sde.tex
\section{Proof of \Cref{thm:well-defined}} \label{app:well-defined}
As preparation before the proof of \cref{thm:well-defined}, we collect and prove some intermediary results. To begin, we shall first prove that $w_\Poi$ is a well-defined stochastic process taking values in $\Sch_{\RadonOp}'$. Recall that $w_\Poi$ is an impulsive white noise that is realized by a Poisson-type random measure of the form
\begin{equation}
    w_\Poi = \sum_{k \in \Z} v_k \, \deltae_{(\vec{w}_k, b_k)},
\end{equation}
where $v_k \overset{\text{i.i.d.}}{\sim} \Pr_V$ for some admissible $\Pr_V$ (in the sense of \cref{defn:admissible}) and the collection of random variables $\paren*{(\vec{w}_k, b_k)}_{k \in \Z}$ is a (homogeneous) Poisson point process on $\cyl$ with rate parameter $\lambda > 0$. This point process satisfies the following properties.
\begin{enumerate}
    \item The $(\vec{w}_k, b_k)$ are mutually independent.
    \item For any measurable subset $\Pi \subset \cyl$, if we define the random variable
    \begin{equation}
        N_\Pi = \abs{\curly{(\vec{w}_k, b_k) \st (\vec{w}_k, b_k) \in \Pi}},
    \end{equation}
    then
    \begin{equation}
        \Pr\paren{N_\Pi = n} = \frac{(\lambda \abs{\Pi})^n}{n!} \e^{-\lambda \abs{\Pi}},
    \end{equation}
    where $\abs{\Pi}$ denotes the $d$-dimensional Hausdorff measure of $\Pi$. That is to say, $N_\Pi$ is a Poisson random variable with mean $\lambda \abs{\Pi}$.
    
    \item For any measurable subset $B \subset \cyl$,
    \begin{equation}
        \Pr\paren*{(\vec{w}_k, b_k) \in B \given (\vec{w}_k, b_k) \in \Pi} = \frac{\abs{B \cap \Pi}}{\abs{\Pi}}.
    \end{equation}
    That is to say, if a point lies in $\Pi$, then its location will be uniformly distributed on $\Pi$.
\end{enumerate}

\begin{lemma} \label[lemma]{lemma:Poi-cf}
    The random measure $w_\Poi$ can be viewed as a measurable mapping
    \begin{equation}
        w_\Poi: (\Omega, \F, \Pr) \to (\Sch_{\RadonOp}', \cB(\Sch_{\RadonOp}'))
    \end{equation}
    with characteristic functional given by
    \begin{equation}
        \hat{\Pr}_{w_\Poi}(\psi) = \exp\paren*{\lambda \int_{\R} \int_{\R} \int_{\Sph^{d-1}} \paren*{\e^{\imag v \psi(\Sphvar, t)} - 1} \dd\Sphvar \dd t \dd \Pr_V(v)},
    \end{equation}
    where $\mathrm{d}\Sphvar$ denotes integration against the surface measure on $\Sph^{d-1}$.
\end{lemma}

\begin{proof}
    Let $\cD(\R^d) \subset \Sch(\R^d)$ denote the space of infinitely differentiable and compactly supported functions on $\R^d$. Let $\cD_{\RadonOp} \coloneqq \RadonOp\paren*{\cD(\R^d)}$ denote the range of the Radon transform on $\cD(\R^d)$. We now summarize the properties of $\cD_{\RadonOp}$ that are relevant for our problem~\citep[cf.,][]{LudwigRadon}. First, $\cD_{\RadonOp}$ is a closed subspace of $\cD(\cyl)$, the nuclear space of infinitely differentiable and compactly supported functions on $\cyl$ and is therefore nuclear. Furthermore, $\cD_{\RadonOp}$ is dense in $\Sch_{\RadonOp}$, which implies that $\Sch_{\RadonOp}'$ is continuously embedded in $\cD_{\RadonOp}'$. In particular, $\cD_{\RadonOp}$ is the subspace of compactly supported functions in $\Sch_{\RadonOp}$.

    Next, we shall prove that $w_{\Poi}$ can be viewed as a measurable mapping
    \begin{equation}
        w_\Poi: (\Omega, \F, \Pr) \to (\cD_{\RadonOp}', \cB_c(\cD_{\RadonOp}'))
    \end{equation}
    by computing its characteristic functional $\hat{\PrQ}_{w_\Poi}$ on $\cD_{\RadonOp}$.\footnote{Note that $\cD_{\RadonOp}$ is not Fréchet so we use the cylindrical $\sigma$-algebra as opposed to the Borel $\sigma$-algebra.} Let $\psi \in \cD_{\RadonOp}$ and let
    \begin{equation}
        N_\psi = \abs{\curly{(\vec{w}_k, b_k) \st (\vec{w}_k, b_k) \in \spt \psi}}.
    \end{equation}
    We have, by definition, that
    \begin{equation}
        \ang{w_{\Poi}, \psi}_{\cD_{\RadonOp}'\times \cD_{\RadonOp}} = \sum_{k = 1}^{N_\psi} v_k' \psi(\vec{w}_k', b_k'),
    \end{equation}
    where we use an appropriate relabeling of $\curly{v_k, \vec{w}_k, b_k \st (\vec{w}_k, b_k) \in \spt \psi}$. Therefore,
    \begin{align*}
        \hat{\PrQ}_{w_\Poi}(\psi)
        &= \E\sq{\e^{\imag\ang{w_\Poi, \psi}_{\cD_{\RadonOp}'\times \cD_{\RadonOp}}}} \\
        &= \E\sq{\e^{\imag\sum_{k = 1}^{N_\psi} v_k' \psi(\vec{w}_k', b_k')}} \\
        &= \E\sq*{\E\sq*{\prod_{k = 1}^{N_\psi}\e^{\imag v_k' \psi(\vec{w}_k', b_k')} \given N_\psi}} \\
        &= \E\sq*{\prod_{k = 1}^{N_\psi}\E\sq*{\e^{\imag v_k' \psi(\vec{w}_k', b_k')} \given N_\psi}} \numberthis \label{eq:indep} \\
        &= \E\sq*{\prod_{k = 1}^{N_\psi}\E\sq*{\e^{\imag v_k' \psi(\vec{w}_k', b_k')}}} \\
        &= \E\sq*{\prod_{k = 1}^{N_\psi}\E\sq*{\E\sq*{\e^{\imag v_k' \psi(\vec{w}_k', b_k')} \given v_k}}} \\
        &= \E\sq*{\prod_{k = 1}^{N_\psi}\E\sq*{\frac{1}{\abs{\spt \psi}} \int_{\spt \psi} \e^{\imag v_k' \psi(\vec{w}, b)} \dd(\vec{w}, b)}} \numberthis \label{eq:uniform} \\
        &= \E\sq*{\prod_{k = 1}^{N_\psi}\frac{1}{\abs{\spt \psi}} \int_\R\int_{\spt \psi} \e^{\imag v \psi(\vec{w}, b)} \dd(\vec{w}, b) \dd\Pr_V(v)}, \numberthis
    \end{align*}
    where \cref{eq:indep} holds by the mutual independence of the $(\vec{w}_k, b_k)$ and \cref{eq:uniform} holds since the random variables 
    \begin{equation}
        (\vec{w}_k', b_k') \given (\vec{w}_k', b_k') \in \spt \psi
    \end{equation}
    are uniformly distributed on $\spt \psi$. Next, define the auxiliary functional
    \begin{equation}
        M(\psi) = \int_\R \int_{\spt \psi} \e^{\imag v \psi(\vec{w}, b)} \dd(\vec{w}, b) \dd\Pr_V(v).
    \end{equation}
    We have that
    \begin{align*}
        \hat{\PrQ}_{w_\Poi}(\psi)
        &= \E\sq*{\prod_{k = 1}^{N_\psi}\frac{M(\psi)}{\abs{\spt \psi}}} \\
        &= \E\sq*{\paren*{\frac{M(\psi)}{\abs{\spt \psi}}}^{N_\psi}} \\
        &= \sum_{n=0}^\infty \paren*{\frac{M(\psi)}{\abs{\spt \psi}}}^n  \frac{(\lambda \abs{\spt \psi})^n}{n!} \e^{-\lambda \abs{\spt \psi}} \numberthis \label{eq:Poisson-rv} \\
        &= \e^{-\lambda \abs{\spt \psi}} \sum_{n=0}^\infty \frac{\paren*{\lambda M(\psi)}^n}{n!}  \\
        &= \e^{-\lambda \abs{\spt \psi}} \e^{\lambda M(\psi)} \numberthis \label{eq:taylor} \\
        &= \exp\paren*{\lambda (M(\psi) - \abs{\spt \psi}} \numberthis \label{eq:int-1-spt} \\
        &= \exp\paren*{\lambda \int_{\R} \int_{\cyl} \paren*{\e^{\imag v \psi(\vec{z})} - 1} \dd\vec{z} \dd\Pr_V(v)}, \numberthis \label{eq:vanish} 
    \end{align*}
    where \cref{eq:Poisson-rv} holds since $N_\psi$ is a Poisson random variable with mean $\lambda \abs{\spt \psi}$, \cref{eq:taylor} holds by the Taylor series expansion of $t \mapsto \e^t$, \cref{eq:int-1-spt} holds since $\abs{\spt \psi} = \int_{\spt \psi} 1 \dd\vec{z}$, and \cref{eq:vanish} holds since $\vec{z} \mapsto \e^{\imag v \psi(\vec{z})} - 1$ vanishes outside $\spt \psi$. 
    At this point, we remark that, since $\Pr_V$ is a L\'evy measure (\cref{defn:admissible}), it is well-known that the form of \cref{eq:vanish} is continuous, positive definite, and satisfies $\hat{\PrQ}_{w_\Poi}(0) = 1$ \citep[see, e.g.,][Theorem~2, p.~275]{gelfand1964generalized}. This implies that $w_\Poi$ is indeed a generalized stochastic process that takes values in $\cD_{\RadonOp}'$.

    To prove the lemma, it remains to extend the domain of $\hat{\PrQ}_{w_\Poi}$ to $\Sch_{\RadonOp}$. To that end, let
    \begin{equation}
        \hat{\Pr}_{w_\Poi}(\psi) = \exp\paren*{\lambda \int_{\R} \int_{\cyl} \paren*{\e^{\imag v \psi(\vec{z})} - 1} \dd\vec{z} \dd\Pr_V(v)}, \quad \psi \in \Sch_{\RadonOp}.
    \end{equation}
    We now invoke an adaption of \citet[Theorem~3]{fageot2014continuity} which investigates impulsive white noise defined on $\R^d$ as a special case. Their theorem implies that, thanks to the admissibility conditions on $\Pr_V$ (\cref{defn:admissible}), the probability measures $\PrQ_{w_\Poi}$ and $\Pr_{w_\Poi}$ are compatible on $\cB(\Sch_{\RadonOp}') = \cB_c(\Sch_{\RadonOp}') \subset \cB_c(\cD_{\RadonOp}')$ in the sense that
    \begin{equation}
        \PrQ_{w_{\Poi}}(B) = \Pr_{w_{\Poi}}(B), \quad \text{for all } B \in \cB(\Sch_{\RadonOp}')
    \end{equation}
    and $\PrQ_{w_\Poi}(\cD_{\RadonOp}' \setminus \Sch_{\RadonOp}') = 0$, which proves the lemma.
\end{proof}

Let $\Sch_\Delta(\R^d) \coloneqq \Delta\paren*{\Sch(\R^d)}$ denote the range of the Laplacian operator on $\Sch(\R^d)$. This is a closed subspace of $\Sch(\R^d)$. Observe that its dual $\Sch_\Delta'(\R^d)$ can be identified with the quotient space $\Sch'(\R^d) / \Null_\Delta$, where
\begin{equation}
    \Null_\Delta = \curly{f \in \Sch'(\R^d) \st \Delta f = 0 \Leftrightarrow \ang{f, \phi}_{\Sch'(\R^d) \times \Sch(\R^d)} = 0 \text{ for all } \phi \in \Sch_\Delta(\R^d)}.
\end{equation}
is the null space of the Laplacian operator. It is well-known that $\Null_\Delta$ is infinite-dimensional and that its members are necessarily polynomials, the so-called \emph{harmonic polynomials}. Therefore, the members of $\Sch_\Delta'(\R^d)$ are actually equivalence classes of the form
\begin{equation}
    [f] = \curly{f + h \st h \in \Null_\Delta} \in \Sch_\Delta'(\R^d),
\end{equation}
where $f \in \Sch'(\R^d)$. With this notation, we now prove \cref{thm:well-defined}.

\begin{proof}[Proof of \Cref{thm:well-defined}]
   Recall that $\TOp_{\ReLU} = \KOp \RadonOp \Delta$ and so $\TOp_{\ReLU}^* = \Delta \RadonOp^* \KOp$. Observe that, by \cref{prop:Radon},
    \begin{equation}
        \TOp_{\ReLU}^*: \Sch_{\RadonOp} \to \Sch_\Delta(\R^d)
    \end{equation}
    is a continuous bijection, where we equip the closed subspaces $\Sch_{\RadonOp} \subset \Sch(\cyl)$ and $\Sch_\Delta(\R^d) \subset \Sch(\R^d)$ with the subspace topology from their respective parent Fréchet spaces. By the open mapping theorem for Fr\'echet spaces~\cite[see, e.g.,][Theorem~2.11]{RudinFA}, there exists a continuous inverse operator
    \begin{equation}
        \TOp_{\ReLU}^{*-}: \Sch_\Delta(\R^d) \to \Sch_{\RadonOp}
    \end{equation}
    with the properties that that $\TOp_{\ReLU}^* \TOp_{\ReLU}^{*-} = \Id$ on $\Sch_\Delta(\R^d)$ and $\TOp_{\ReLU}^{*-}\TOp_{\ReLU}^* = \Id$ on $\Sch_{\RadonOp}$. Therefore, by duality, we have the continuous bijections
    \begin{align*}
        &\TOp_{\ReLU}: \Sch_\Delta'(\R^d) \to \Sch_{\RadonOp} \\
        &\TOp_{\ReLU}^{-}: \Sch_{\RadonOp}' \to \Sch_\Delta'(\R^d), \numberthis
    \end{align*}
    where we recall that $\Sch_\Delta'(\R^d) \cong \Sch'(\R^d) / \Null_\Delta$.
    
    Next, we note that the operator 
    \begin{equation}
        \TOp_{\ReLU}^{\dagger\varepsilon*}: \varphi \mapsto \int_{\R^d} k^\varepsilon_\vec{x}(\dummy) \varphi(\vec{x}) \dd\vec{x} 
    \end{equation}
    specified in \cref{eq:inverse-op-adjoint} continuously maps $\Sch_\Delta(\R^d) \to \Sch_{\RadonOp}$ \citep[cf.,][Equation~(A.3)]{parhi2025function}. Observe that, by \cref{prop:inverse}, its extension by duality $\TOp_{\ReLU}^{\dagger\varepsilon}: \Sch_{\RadonOp}' \to \Sch_\Delta'(\R^d)$ coincides with $\TOp_{\ReLU}^{-}$. In particular, $\TOp_{\ReLU}^{\dagger\varepsilon}$ imposes the boundary conditions from \cref{eq:boundary} on the affine component of the harmonic polynomials in the equivalence classes in $\Sch_\Delta'(\R^d)$. Said differently, the range space $\TOp_{\ReLU}^{\dagger\varepsilon}\paren*{\Sch_{\RadonOp}'}$ is the closed subspace of $\Sch_\Delta'(\R^d)$ whose equivalence class members $[s] \in \Sch_\Delta'(\R^d)$ additionally satisfy 
    \begin{equation}
        \sq*{(\partial^\vec{m} g_d^\varepsilon) * s_0}(\vec{0}) = 0, \abs{\vec{m}} \leq 1 \label{eq:boundary-proof}
    \end{equation}
    for all $s_0 \in [s]$. Therefore, we can rewrite the SDE \cref{eq:SDE-thm-mol} as
    \begin{equation}
        s \overset{\Law}{=} \TOp_{\ReLU}^{\dagger\varepsilon} w_\Poi, \label{eq:SDE-rewrite}
    \end{equation}
    where the equality is understood in $\Sch_\Delta'(\R^d)$, i.e.,
    \begin{equation}
        \ang{s, \phi}_{\Sch_\Delta'(\R^d) \times \Sch_\Delta(\R^d)} \overset{\Law}{=} \ang{\TOp_{\ReLU}^{\dagger\varepsilon} w_\Poi, \phi}_{\Sch_\Delta'(\R^d) \times \Sch_\Delta(\R^d)} = \ang{w_\Poi, \TOp_{\ReLU}^{\dagger\varepsilon*} \phi}_{\Sch_{\RadonOp}' \times \Sch_{\RadonOp}},
    \end{equation}
    for all $\phi \in \Sch_\Delta(\R^d)$. The above equality implies that the characteristic functional of \emph{any solution} $s$ to \cref{eq:SDE-rewrite} (and, subsequently, the original SDE \cref{eq:SDE-thm-mol}) takes the form
    \begin{equation}
        \hat{\Pr}_s(\phi) = \hat{\Pr}_{\TOp_{\ReLU}^{\dagger\varepsilon} w_\Poi}(\phi) = \hat{\Pr}_{w_\Poi}(\TOp_{\ReLU}^{\dagger\varepsilon*} \phi). \label{eq:cf-quotient}
    \end{equation}
    This characteristic functional is well-defined for any $\phi \in \Sch_\Delta(\R^d)$ since $\TOp_{\ReLU}^{\dagger\varepsilon*} \phi \in \Sch_{\RadonOp}$, which ensures that the right-hand side is well-defined by \cref{lemma:Poi-cf}.
    
    Since $s_{\ReLU}^\varepsilon \coloneqq \TOp_{\ReLU}^{\dagger\varepsilon} w_\Poi$ via the computation in \cref{eq:Poi-to-ReLU}, we see that $s_{\ReLU}^\varepsilon$ is one member in an equivalence class in $\Sch'(\R^d) / \Null_\Delta$. In particular, this implies that $s_{\ReLU} \in \Sch'(\R^d)$ and that the equivalence class $[s_{\ReLU}^\varepsilon] = \curly{s_{\ReLU}^\varepsilon + h \st h \in \Null_\Delta}$ is a well-defined stochastic process that takes values in $\Sch_{\Delta}'(\R^d) \cong \Sch'(\R^d) / \Null_\Delta$ whose characteristic functional on $\Sch_\Delta$ is given by \cref{eq:cf-quotient}. Equivalently stated, the full set of tempered weak solutions of the SDE \cref{eq:SDE-thm-mol} has members that necessarily take the form $s_{\ReLU}^\varepsilon + h$, where $h \in \Null_\Delta$ is a harmonic polynomial of degree $\geq 2$ (since boundary conditions of the SDE, imposed by $\TOp_{\ReLU}^{\dagger\varepsilon}$, force the affine component of all solutions to be the same). Consequently, from these boundary conditions, we readily see that the only CPwL solution to the SDE is $s_{\ReLU}^\varepsilon$.

    To complete the proof we need to derive the form of the characteristic functional of $s_{\ReLU}^\varepsilon$ on the larger space $\Sch(\R^d) \supset \Sch_\Delta(\R^d)$. For any $\varphi \in \Sch(\R^d)$, we have that
    \begin{equation}
        \ang{s_{\ReLU}^\varepsilon, \varphi}_{\Sch'(\R^d) \times \Sch(\R^d)} \overset{\Law}{=} \ang{\TOp_{\ReLU}^{\dagger\varepsilon} w_\Poi, \varphi}_{\Sch'(\R^d) \times \Sch(\R^d)} \label{eq:asdf}
    \end{equation}
    From the expression of the kernel $(\Sphvar, t) \mapsto k^\varepsilon_\vec{x}(\Sphvar, t)$ in \cref{eq:kernel-op} we see that (i) it is continuous in the variables $(\Sphvar, t) \in \cyl$ and (ii) it decays faster than any polynomial in the $t$-variable. Therefore, for every $\varphi \in \Sch(\R^d)$, the function $\TOp_{\ReLU}^{\dagger\varepsilon*}\curly{\varphi}$ is a continuous function in $(\Sphvar, t) \in \cyl$ that decays faster than any polynomial in the $t$-variable. In particular, this ensures that, for any $1 \leq p \leq \infty$, the map
    \begin{equation}
        \TOp_{\ReLU}^{\dagger\varepsilon*}: \Sch(\R^d) \to L^p(\cyl)
        \label{eq:Lp-eps}
    \end{equation}
    is continuous.
    
    The right-hand side of \cref{eq:asdf} is, by definition, the integration of $\TOp_{\ReLU}^{\dagger\varepsilon*}\curly{\varphi}$ against the locally finite Radon measure $w_\Poi$, i.e., for any $\varphi \in \Sch(\R^d)$ we have that
    \begin{align*}
        \ang{s_{\ReLU}^\varepsilon, \varphi}_{\Sch'(\R^d) \times \Sch(\R^d)}
        &\overset{\mathclap{\Law}}{=} \int_{\cyl} \TOp_{\ReLU}^{\dagger\varepsilon*}\curly{\varphi} \dd\paren*{\sum_{k \in \Z} v_k \, \deltae_{(\vec{w}_k, b_k)}} \\
        &= \sum_{k \in \Z} v_k \int_{\cyl} \TOp_{\ReLU}^{\dagger\varepsilon*}\curly{\varphi} \dd\deltae_{(\vec{w}_k, b_k)} \\
        &= \sum_{k \in \Z} v_k \TOp_{\ReLU}^{\dagger\varepsilon*}\curly{\varphi}(\vec{w}_k, b_k), \numberthis \label{eq:sum-L1}
    \end{align*}
    where interchanging of the integral and sum in the second line is well-defined due to the regularity of $\TOp_{\ReLU}^{\dagger\varepsilon*}\curly{\varphi}$ and the third line uses the fact that the range of $\TOp_{\ReLU}^{\dagger\varepsilon*}$ on $\Sch(\R^d)$ is a space of even functions. This proves that
    \begin{align*}
        &\phantom{{}={}}\hat{\Pr}_{s_{\ReLU}^\varepsilon}(\varphi) \\
        &= \hat{\Pr}_{w_\Poi}(\TOp_{\ReLU}^{\dagger\varepsilon*} \varphi) = \exp\paren*{\lambda \int_{\R} \int_\R \int_{\Sph^{d-1}} \paren*{\e^{\imag v \TOp_{\ReLU}^{\dagger\varepsilon*}\curly{\varphi}(\Sphvar, t)} - 1} \dd\Sphvar \dd t \dd\Pr_V(v)}, \numberthis \label{eq:cf-form-proof}
    \end{align*}
    for all $\varphi \in \Sch(\R^d)$, where the last equality comes from \cref{lemma:Poi-cf}. We shall now verify that $\hat{\Pr}_{s_{\ReLU}^\varepsilon}$ is a valid characteristic functional on $\Sch(\R^d)$. This then implies that $s_{\ReLU}^\varepsilon: (\Omega, \F, \Pr) \to (\Sch'(\R^d), \cB(\Sch'(\R^d))$ is a measurable mapping and therefore a well-defined stochastic process. 
    
    Observe that the second admissibility condition on $\Pr_V$ (\cref{item:FAM} in \cref{defn:admissible}) states that $\Pr_V$ has a finite absolute moment. This is a sufficient condition to ensure that this characteristic functional \cref{eq:cf-form-proof} is well-defined for every $\varphi \in \Sch(\R^d)$. Indeed, we have that
    \begin{equation}
        \Psi(\xi) \coloneqq \lambda \int_{\R} \paren*{\e^{\imag v\xi} - 1} \dd\Pr_V(v) \leq \lambda \, \abs{\xi} \, \E\sq{\abs{V}}, \label{eq:Psi}
    \end{equation}
    where $V \sim \Pr_V$~\citep[cf.,][p.~1952]{unser2014unified}. Therefore,
    \begin{align*}
        \hat{\Pr}_{s_{\ReLU}^\varepsilon}(\varphi)
        &= \exp\paren*{\lambda \int_{\R} \int_\R \int_{\Sph^{d-1}} \paren*{\e^{\imag v \TOp_{\ReLU}^{\dagger\varepsilon*}\curly{\varphi}(\Sphvar, t)} - 1} \dd\Sphvar \dd t \dd\Pr_V(v)} \\
        &= \exp\paren*{\int_{\cyl} \Psi\paren*{\TOp_{\ReLU}^{\dagger\varepsilon*}\curly{\varphi}} \dd\vec{z}} \\
        &\leq \exp\paren*{C\, \norm{\TOp_{\ReLU}^{\dagger\varepsilon*}\curly{\varphi}}_{L^1}} \\
        &< \infty, \numberthis
    \end{align*}
    for any $\varphi \in \Sch(\R^d)$, where $C = \lambda \, \E\sq{\abs{V}} < \infty$, where in the last line we used \cref{eq:Lp-eps} with $p=1$. Since $\TOp_{\ReLU}^{\dagger\varepsilon*}: \Sch(\R^d) \to L^1(\cyl)$ linearly and continuously, Proposition~3.1 of \citet{fageot2019scaling} then guarantees that $\hat{\Pr}_{s_{\ReLU}^\varepsilon}$ is continuous, positive definite, and satisfies $\hat{\Pr}_{s_{\ReLU}^\varepsilon}(0) = 1$. Therefore, the Bochner--Minlos theorem ensures that $\hat{\Pr}_{s_{\ReLU}^\varepsilon}$ is the characteristic functional of the well-defined stochastic process $s_{\ReLU}^\varepsilon$.

    In the limiting scenario of $\varepsilon \to 0$, we see that the random neural network $s_{\ReLU} \sim \ReLUP(\lambda; \Pr_V)$ is a measurable mapping $(\Omega, \F, \Pr) \to (\Sch'(\R^d), \cB(\Sch'(\R^d))$
    whose characteristic functional is
    \begin{equation}
        \hat{\Pr}_{s_{\ReLU}}(\varphi) = \exp\paren*{\lambda \int_{\R} \int_\R \int_{\Sph^{d-1}} \paren*{\e^{\imag v \TOp_{\ReLU}^{\dagger*}\curly{\varphi}(\Sphvar, t)} - 1} \dd\Sphvar \dd t \dd\Pr_V(v)}, \quad \varphi \in \Sch(\R^d),
    \end{equation}
    where we observe that this limiting characteristic functional remains to be valid in the sense of the Bochner--Minlos theorem since the property that, for any $1 \leq p \leq \infty$, the map
    \begin{equation}
        \TOp_{\ReLU}^{\dagger*}: \Sch(\R^d) \to L^p(\cyl)
        \label{eq:Lp}
    \end{equation}
    is continuous, remains to be true since $k_\vec{x} = \lim_{\varepsilon \to 0} k^\varepsilon_\vec{x}$ (pointwise limt) is compactly supported. Consequently, $s_{\ReLU}$ is the unique CPwL solution to SDE \cref{eq:SDE-thm-new}.\footnote{The mollifier argument is necessary in order to make sense of the boundary conditions \cref{eq:SDE-thm-new} for elements of $\Sch'(\R^d)$ that are not regular enough for the derivatives to exist.}
\end{proof}

%% file: appendix/properties.tex
\section{Proof of \cref{th:properties}}
\label{app:properties}

\begin{proof}
    \begin{enumerate}
        \item Thanks to the moment generating properties of the characteristic functional \citep{gelfand1964generalized}, the mean functional of $s_{\ReLU}$ can be obtained as
\begin{equation}
    \mu_{s_{\ReLU}}(\varphi) = \E\sq{\ang{s_{\ReLU}, \varphi}_{\Sch'(\R^d) \times \Sch(\R^d)}} = (-\imag) \frac{\dd}{\dd \xi} \hat{\Pr}_{s_{\ReLU}}(\xi \varphi) \Big\rvert_{\xi=0},
\end{equation}
where $\varphi \in \Sch(\R^d)$. First, observe that the characteristic functional of $s_{\ReLU}$ can be written as
\begin{equation}
    \hat{\Pr}_{s_{\ReLU}}(\varphi) = \exp\left( \int_{\cyl} \Psi\left(\TOp_{\ReLU}^{\dagger*}\curly{\varphi}(\vec{z})\right) \dd\vec{z} \right)
\end{equation}
with $\Psi$ defined as in \cref{eq:Psi}. Here, note that we have
\begin{equation}
    \Psi'(x) = \imag \lambda \int_{\R} v \e^{\imag vx} \dd\Pr_V(v).
\end{equation}
Let us denote $h(\vec{z}) = \TOp_{\ReLU}^{\dagger*}\curly{\varphi}(\vec{z})$. By applying the chain rule, we can write
\begin{equation}
    \de{\xi} \hat{\Pr}_{s_{\ReLU}}(\xi \varphi) = \exp\left( \int_{\R \times \Sph^{d-1}} \Psi\left(\xi h(\vec{z})\right) \dd\vec{z} \right) \cdot \int_{\R \times \Sph^{d-1}} \Psi'\left(\xi h(\vec{z})\right) h(\vec{z}) \dd\vec{z}.
\end{equation}
On setting $\xi = 0$, we get
\begin{equation}
    \de{\xi} \hat{\Pr}_{s_{\ReLU}}(\xi \varphi) \Big\rvert_{\xi=0} = \Psi'(0) \int_{\R \times \Sph^{d-1}} h(\vec{z}) \dd\vec{z}
\end{equation}
as $\Psi(0) = 0$. Therefore, the mean functional is
\begin{align*}
    \mu_{s_{\ReLU}}(\varphi) &= (-\imag) \frac{\dd}{\dd \xi} \hat{\Pr}_{s_{\ReLU}}(\xi \varphi) \Big\rvert_{\xi=0} \\
    &= \lambda \E\sq{V} \int_{\cyl} \TOp_{\ReLU}^{\dagger*}\curly{\varphi}(\vec{z}) \dd\vec{z} \\
    &= \lambda \E\sq{V} \int_{\R^d} \int_{\R} \int_{\Sph^{d-1}} k_{\vec{x}}(\vec{u}, t) \varphi(\vec{x}) \dd \vec{u} \dd t \dd \vec{x}. \numberthis \label{eq:dervn_mean_functional}
\end{align*}
Next, we establish a link between the mean functional of $s_{\ReLU}$ and the quantity $\E\sq{s_{\ReLU}(\vec{x})}$. Since $s_{\ReLU}$ has a pointwise interpretation, we have
\begin{equation}
    \ang{s_{\ReLU}, \varphi}_{\Sch'(\R^d) \times \Sch(\R^d)} = \int_{\R^d} s_{\ReLU}(\vec{x}) \varphi(\vec{x}) \dd \vec{x}.
\end{equation}
Consequently, the mean functional can also be computed as
\begin{align*}
    \mu_{s_{\ReLU}}(\varphi) &= \E\sq{\ang{s_{\ReLU}, \varphi}_{\Sch'(\R^d) \times \Sch(\R^d)}} = \E\sq*{\int_{\R^d} s_{\ReLU}(\vec{x}) \varphi(\vec{x}) \dd \vec{x}} \\
    &= \int_{\R^d} \E\sq{s_{\ReLU}(\vec{x})} \varphi(\vec{x}) \dd \vec{x}, \numberthis \label{eq:dervn_mean_functional_pw}
\end{align*}
where exchanging the expectation and the integral is justified by the Fubini--Tonelli theorem since the integrand in \cref{eq:dervn_mean_functional} is absolutely integrable by \cref{eq:Lp} with $p = 1$. On comparing \cref{eq:dervn_mean_functional_pw} with \cref{eq:dervn_mean_functional}, we see that 
\begin{equation}
    \E\sq{s_{\ReLU}(\vec{x})} = \lambda \E\sq{V} \int_{\R} \int_{\Sph^{d-1}} k_{\vec{x}}(\vec{u}, t) \dd\vec{u} \dd t.
\end{equation}

\item The covariance functional of $s_{\ReLU}$ is given by
\begin{align*}
    &\phantom{{}={}}\Sigma_{s_{\ReLU}}(\varphi_1, \varphi_2) \\
    &= \E\sq[\Big]{\left(\ang{s_{\ReLU}, \varphi_1}_{\Sch'(\R^d) \times \Sch(\R^d)} - \mu_{s_{\ReLU}}(\varphi_1)\right) \left(\ang{s_{\ReLU}, \varphi_2}_{\Sch'(\R^d) \times \Sch(\R^d)} - \mu_{s_{\ReLU}}(\varphi_2)\right)} \\
    &= \mathcal{R}_{s_{\ReLU}}(\varphi_1, \varphi_2) - \mu_{s_{\ReLU}}(\varphi_1)\mu_{s_{\ReLU}}(\varphi_2), \numberthis
\end{align*}
where $\varphi_1, \varphi_2 \in \Sch(\R^d)$ and 
\begin{equation}
    \mathcal{R}_{s_{\ReLU}}(\varphi_1, \varphi_2) = \E\sq*{\ang{s_{\ReLU}, \varphi_1}_{\Sch'(\R^d) \times \Sch(\R^d)} \ \ang{s_{\ReLU}, \varphi_2}_{\Sch'(\R^d) \times \Sch(\R^d)}}
\end{equation}
is the correlation functional of $s_{\ReLU}$. This quantity can be computed from its characteristic functional \citep[cf.,][]{gelfand1964generalized} as
\begin{equation}
    \mathcal{R}_{s_{\ReLU}}(\varphi_1, \varphi_2) = -\frac{\ddd^{2}}{\ddd \xi_1 \ddd \xi_2} \hat{\Pr}_{s_{\ReLU}}(\xi_1 \varphi_1 + \xi_2 \varphi_2) \Big\rvert_{\xi_1=0, \xi_2=0}.
\end{equation}
Let us first define the quantity $f(\xi_1, \xi_2)$ as 
\begin{align*}
    f(\xi_1, \xi_2) &= \int_{\cyl} \Psi\left(\TOp_{\ReLU}^{\dagger*}\curly{\xi_1\varphi_1 + \xi_2 \varphi_2}(\vec{z})\right) \dd\vec{z} \\
    &= \int_{\cyl} \Psi\left(\xi_1 \TOp_{\ReLU}^{\dagger*}\curly{\varphi_1}(\vec{z}) + \xi_2 \TOp_{\ReLU}^{\dagger*}\curly{\varphi_2}(\vec{z})\right) \dd\vec{z}. \numberthis
\end{align*}
Further, let us denote $h_1(\vec{z}) = \TOp_{\ReLU}^{\dagger*}\curly{\varphi_1}(\vec{z})$ and $h_2(\vec{z}) = \TOp_{\ReLU}^{\dagger*}\curly{\varphi_2}(\vec{z})$. By applying the chain rule twice, we write
\begin{align*}
    \frac{\ddd^{2}}{\ddd \xi_1 \ddd \xi_2} \hat{\Pr}_{s_{\ReLU}}(\xi_1 \varphi_1 &+ \xi_2 \varphi_2) \\
    &= \exp(f(\xi_1, \xi_2)) \left( \frac{\ddd^2}{\ddd \xi_1 \ddd \xi_2} f(\xi_1, \xi_2) + \frac{\ddd}{\ddd \xi_1} f(\xi_1, \xi_2) \frac{\ddd}{\ddd \xi_2} f(\xi_1, \xi_2)\right), \numberthis
\end{align*}
where
\begin{equation}
    \frac{\ddd}{\ddd \xi_k} f(\xi_1, \xi_2) = \int_{\cyl} \Psi'\left(\xi_1 h_1(\vec{z}) + \xi_2 h_2(\vec{z})\right) h_{k}(\vec{z}) \dd \vec{z}
\end{equation}
and 
\begin{equation}
    \frac{\ddd^2}{\ddd \xi_1 \ddd \xi_2} f(\xi_1, \xi_2) = \int_{\cyl} \Psi''\left(\xi_1 h_1(\vec{z}) + \xi_2 h_2(\vec{z})\right) h_1(\vec{z}) h_2(\vec{z}) \dd \vec{z}. 
\end{equation}
On setting $\xi_1 = 0$ and $\xi_2 = 0$, we get
\begin{align*}
    \frac{\ddd^{2}}{\ddd \xi_1 \ddd \xi_2} \hat{\Pr}_{s_{\ReLU}}(\xi_1 \varphi_1 + \xi_2 \varphi_2) &\Big\rvert_{\xi_1=0, \xi_2=0} = \ \Psi''(0) \int_{\cyl} h_1(\vec{z}) h_2(\vec{z}) \dd \vec{z} \\
    &+ \left( \Psi'(0) \int_{\cyl} h_1(\vec{z}) \ddd\vec{z}\right)\left( \Psi'(0) \int_{\cyl} h_2(\vec{z}) \dd\vec{z}\right). \numberthis
\end{align*}
Note that we have
\begin{equation}
    \Psi''(x) = - \lambda \int_{\R} v^2 \e^{\imag vx} \dd\Pr_V(v).
\end{equation}
Thus, the correlation functional is of the form
\begin{align*}
    \mathcal{R}_{s_{\ReLU}}(\varphi_1, \varphi_2) &= -\frac{\ddd^{2}}{\ddd \xi_1 \ddd \xi_2} \hat{\Pr}_{s_{\ReLU}}(\xi_1 \varphi_1 + \xi_2 \varphi_2) \Big\rvert_{\xi_1=0, \xi_2=0} \\
    &= \lambda \E\sq{V^2} \left(\int_{\cyl} \TOp_{\ReLU}^{\dagger*}\curly{\varphi_1}(\vec{z}) \TOp_{\ReLU}^{\dagger*}\curly{\varphi_2}(\vec{z}) \dd\vec{z} \right) \\
    & \ \ \ \ + \mu_{s_{\ReLU}}(\varphi_1) \mu_{s_{\ReLU}}(\varphi_2) \numberthis.
\end{align*}
Consequently, the covariance functional is given by
\begin{align*}
    &\phantom{{}={}}\Sigma_{s_{\ReLU}}(\varphi_1, \varphi_2) \\
    &= \lambda \E\sq{V^2} \int_{\cyl} \TOp_{\ReLU}^{\dagger*}\curly{\varphi_1}(\vec{z}) \TOp_{\ReLU}^{\dagger*}\curly{\varphi_2}(\vec{z}) \dd\vec{z} \\
    &= \lambda \E\sq{V^2} \int_{\R^d} \int_{\R^d} \int_{\R} \int_{\Sph^{d-1}} k_{\vec x}(\vec u, t) k_{\vec y}(\vec u, t) \varphi_1(\vec x) \varphi_2(\vec y) \dd \vec{u} \dd t \dd \vec{x} \dd \vec{y}. \numberthis \label{eq:dervn_cov_functional}
\end{align*}
Next, we derive the connection between the covariance functional of $s_{\ReLU}$ and the autocovariance $\E\sq{(s_{\ReLU}(\vec{x}) - \E\sq{s_{\ReLU}(\vec{x})}) (s_{\ReLU}(\vec{y}) - \E\sq{s_{\ReLU}(\vec{y})})}$. Since $s_{\ReLU}$ has a pointwise interpretation, the covariance functional can also be computed as 
\begin{align*}
        &\phantom{{}={}}\Sigma_{s_{\ReLU}}(\varphi_1, \varphi_2) \\
        &= \E\sq[\Big]{\left(\ang{s_{\ReLU}, \varphi_1}_{\Sch'(\R^d) \times \Sch(\R^d)} - \mu_{s_{\ReLU}}(\varphi_1)\right)
        \left(\ang{s_{\ReLU}, \varphi_2}_{\Sch'(\R^d) \times \Sch(\R^d)} - \mu_{s_{\ReLU}}(\varphi_2)\right)} \\
        &= \E\left[\left(\int_{\R^d} \left(s_{\ReLU}(\vec{x}) - \E\sq{s_{\ReLU}(\vec{x})} \right) \varphi_1(\vec{x}) \dd \vec{x}\right) \left(\int_{\R^d} \left(s_{\ReLU}(\vec{y}) - \E\sq{s_{\ReLU}(\vec{y})} \right) \varphi_2(\vec{y}) \dd \vec{y}\right) \right] \\
        &= \int_{\R^d}\int_{\R^d} \E\sq{(s_{\ReLU}(\vec{x}) - \E\sq{s_{\ReLU}(\vec{x})}) (s_{\ReLU}(\vec{y}) - \E\sq{s_{\ReLU}(\vec{y})})} \varphi_1(\vec{x}) \varphi_2(\vec{y}) \dd \vec{x} \dd \vec{y}, \numberthis \label{eq:dervn_cov_functional_pw}
    \end{align*}
where exchanging the expectation and the integral is justified by the Fubini--Tonelli theorem since the integrand in \cref{eq:dervn_cov_functional} is absolutely integrable from \cref{eq:Lp} with $p = 2$. If we compare \cref{eq:dervn_cov_functional_pw} with \cref{eq:dervn_cov_functional}, we see that
\begin{align*}
    C_{s_{\ReLU}}(\vec x, \vec y) &= \E\sq{(s_{\ReLU}(\vec{x}) - \E\sq{s_{\ReLU}(\vec{x})}) (s_{\ReLU}(\vec{y}) - \E\sq{s_{\ReLU}(\vec{y})})} \\
    &= \lambda \E\sq{V^2} \int_{\R} \int_{\Sph^{d-1}} k_{\vec{x}}(\vec{u},t) k_{\vec{y}}(\vec{u},t) \dd \vec{u} \dd t. \numberthis \label{eq:dervn_cov_integral_form}
\end{align*}
To simplify the double integral in \cref{eq:dervn_cov_integral_form}, we first observe that, by definition,
\begin{equation}
    (\vec{x}, \vec{y}) \mapsto \int_{\R} \int_{\Sph^{d-1}} k_{\vec{x}}(\vec{u},t) k_{\vec{y}}(\vec{u},t) \dd \vec{u} \dd t, \quad (\vec{x}, \vec{y}) \in \R^d \times \R^d,
        \label{eq:double-int}
\end{equation}
is the (Schwartz) kernel of the operator $\TOp_{\ReLU}^{\dagger}\TOp_{\ReLU}^{\dagger*}$. Next, we note that the right-inverse operator can be equivalently specified as the composition of operators $\TOp_{\ReLU}^{\dagger} = (\Id - \P)\Delta^{-1}\RadonOp^*$ \citep[cf.][Equation~(57)]{UnserRidges}, where $\Delta^{-1}$ is the Riesz potential of order $2$, i.e., it is the Fourier multiplier
\begin{equation}
    \reallywidehat{(-\Delta)^{-\frac{\gamma}{2}} f}(\vec{\omega}) = \norm{\vec{\omega}}_2^{-\gamma} \, \hat{f}(\vec{\omega}), \quad \vec{\omega} \in \R^d,
\end{equation}
with $\gamma = 2$,
and $\P$ is the projection onto the space of affine functions adapted to the boundary conditions of the SDE \cref{eq:SDE-thm-new}. Concretely,
\begin{equation}
    \P\curly{f} = \sum_{n=0}^d \ang{\phi_n, f} p_n,
\end{equation}
where $p_0(\vec{x}) = 1$ and $p_n(\vec{x}) = x_n$, $n = 1, \ldots, d$ is a basis for the space of affine function on $\R^d$ and $\phi_0 = \delta$ (Dirac distribution) and $\phi_n = -\delta_n' \coloneqq -\partial_{x_n}\delta$, $n = 1, \ldots, d$, is the linear functional that evaluates the partial derivative in the $n$th component at $\vec{0}$, i.e., $\ang{\phi_n, f} = \partial_{x_n} f(\vec{0})$, $n = 1, \ldots, d$. Consequently, the adjoint projector is given by
\begin{equation}
    \P^*\curly{f} = \sum_{n=0}^d \ang{p_n, f} \phi_n.
\end{equation}
With this notation, we have that
\begin{align*}      \TOp_{\ReLU}^{\dagger}\TOp_{\ReLU}^{\dagger*}
    &= (\Id - \P)\Delta^{-1}\RadonOp^*\RadonOp\Delta^{-1}(\Id - \P^*) \\
    &= (\Id - \P)\Delta^{-1}(-\Delta)^{-\frac{d-1}{2}}\Delta^{-1}(\Id - \P^*) \\
    &= (\Id - \P)(-\Delta)^{-\frac{d+3}{2}}(\Id - \P^*), \numberthis
\end{align*}
where the second line follows from \cref{prop:Radon}. The (Schwartz) kernel of the operator (generalized impulse response) can be identified with $(\vec{x}, \vec{y}) \mapsto \TOp_{\ReLU}^{\dagger}\TOp_{\ReLU}^{\dagger*}\curly{\delta(\dummy - \vec{y})}(\vec{x})$. We have that
\begin{equation}
    (\Id - \P^*)\curly{\delta(\dummy - \vec{y})} = \delta(\dummy - \vec{y}) - \sum_{k=0}^d \ang{p_k, \delta(\dummy - \vec{y})} \phi_n = \delta(\dummy - \vec{y}) - \delta + \sum_{n=1}^d y_1 \delta_n',
\end{equation}
where we used the property that the shifted Dirac distribution is the sampling functional. Next,
\begin{align*}
    (-\Delta)^{-\frac{d+3}{2}} (\Id - \P^*)\curly{\delta(\dummy - \vec{y})}(\vec{x})
    &= A\paren*{\norm{\vec{x} - \vec{y}}_2^3 - \norm{\vec{x}}_2^3 + \sum_{n=1}^d y_1 (3 x_n \norm{\vec{x}}_2)} \\
    &= A\paren*{\norm{\vec{x} - \vec{y}}_2^3 - \norm{\vec{x}}_2^3 + 3 \vec{x}^\T\vec{y} \norm{\vec{x}}_2}, \numberthis
\end{align*}
where $A = \frac{\Gamma(-3/2)}{2^{d+3} \pi^{d/2} \Gamma((d+3)/2)}$ and we used the fact that $\vec{x} \mapsto A \norm{\vec{x}}_2^3$ is the radially symmetric Green's function of $(-\Delta)^{\frac{d+3}{2}}$~\citep{GelfandV1}. Finally,
\begin{align*}
    &\phantom{{}={}} (\Id - \P) (-\Delta)^{-\frac{d+3}{2}} (\Id - \P^*) \curly{\delta(\dummy - \vec{y})}(\vec{x}) \\
    &= A\paren*{\norm{\vec{x} - \vec{y}}_2^3 - \norm{\vec{x}}_2^3 + 3 \vec{x}^\T\vec{y} \norm{\vec{x}}_2} - A \paren*{\norm{\vec{y}}_2^3 - \sum_{n=1}^d 3 y_n \norm{\vec{y}}_2 x_n} \\
    &= A\paren*{\norm{\vec{x} - \vec{y}}_2^3 - \norm{\vec{x}}_2^3 - \norm{\vec{y}}_2^3 + 3 \vec{x}^\T\vec{y}\paren*{\norm{\vec{x}}_2 + \norm{\vec{y}}_2}}. \numberthis
\end{align*}
Putting everything together, we find that the autocovariance takes the form
\begin{equation}
    C_{s_{\ReLU}}(\vec{x}, \vec{y}) = \lambda A \E\sq{V^2} \paren*{\norm{\vec{x} - \vec{y}}_2^3 - \norm{\vec{x}}_2^3 - \norm{\vec{y}}_2^3 + 3 \vec{x}^\T\vec{y}\paren*{\norm{\vec{x}}_2 + \norm{\vec{y}}_2}}.
\end{equation}

\item In order to show that $s_{\ReLU}$ is isotropic, we will show that its characteristic functional satisfies
    \begin{equation}
        \hat{\Pr}_{s_{\ReLU}}(\varphi) = \hat{\Pr}_{s_{\ReLU}}\left(\varphi\left(\mat{U} \cdot \right)\right)
    \end{equation}
    for any $\varphi \in \Sch(\R^d)$ and any $(d \times d)$ rotation matrix $\mat{U}$. First, we note that the kernel of $\TOp_{\ReLU}^{\dagger}$ can be written as
\begin{align*}
    k_\vec{x}(\vec{u}, t)
        &= \ReLU(\vec{u}^\T\vec{x} - t) - \frac{(\vec{u}^\T\vec{x} - t)}{2} - \frac{\abs{t}}{2} + (\vec{u}^\T\vec{x})\frac{\sgn(t)}{2} \\
        &= \ReLU(\vec{u}^\T\vec{x} - t) + (\vec{u}^\T\vec{x}) h_1(t) + h_2(t), \numberthis
\end{align*}
where $h_1(t) = \frac{\sgn(t)-1}{2}$ and $h_2(t) = \frac{t-\abs{t}}{2}$. Let $\mat{U}$ be a $(d \times d)$ rotation matrix. Then, we have
\begin{align}
\TOp_{\ReLU}^{\dagger*}\curly{\varphi(\mat{U} \cdot)}(\vec{u}, t) &= \int_{\R^d} k_{\vec{x}}(\vec{u}, t) \varphi(\mat{U} \vec{x}) \dd\vec{x} \nonumber \\
    &= \int_{\R^d} k_{\mat{U}^\T \vec{\tilde{x}}}(\vec{u}, t) \varphi(\vec{\tilde{x}}) \dd\vec{\tilde{x}} \label{eq:t1} \\
    &= \int_{\R^d} k_{\vec{\tilde{x}}}(\mat{U} \vec{u}, t) \varphi(\vec{\tilde{x}}) \dd\vec{\tilde{x}} \label{eq:t2} \\
    &= \TOp_{\ReLU}^{\dagger*}\curly{\varphi}(\mat{U} \vec{u}, t) \label{eq:t3}.
\end{align}
The transition from \cref{eq:t1} to \cref{eq:t2} is possible because
\begin{align*}
    k_{\mat{U}^\T \vec{\tilde{x}}}(\vec{u}, t) &= \ReLU(\vec{u}^\T \mat{U}^\T \vec{\tilde{x}} - t) + (\vec{u}^\T \mat{U}^\T \vec{\tilde{x}}) h_1(t) + h_2(t) \\
    &= \ReLU((\mat{U} \vec{u})^\T \vec{\tilde{x}} - t) + ((\mat{U} \vec{u})^\T \vec{\tilde{x}}) h_1(t) + h_2(t) \\
    &= k_\vec{\tilde{x}}(\mat{U} \vec{u}, t).
\end{align*}
From \cref{thm:well-defined}, the characteristic functional of $s_{\ReLU}$ is given by
\begin{equation}
    \hat{\Pr}_{s_{\ReLU}}(\varphi) = \exp\left( \int_{\R} \int_{\Sph^{d-1}} \Psi\left(\TOp_{\ReLU}^{\dagger*}\curly{\varphi}(\vec{u}, t)\right) \dd\vec{u} \dd t \right) \label{eq:t4}
\end{equation}
with $\Psi$ defined as in \cref{eq:Psi}. Thus, based on \cref{eq:t4,eq:t3}, we can write
\begin{align}
    \hat{\Pr}_{s_{\ReLU}}(\varphi(\mat{U} \cdot)) &= \exp\left( \int_{\R} \int_{\Sph^{d-1}} \Psi\left(\TOp_{\ReLU}^{\dagger*}\curly{\varphi(\mat{U} \cdot)}(\vec{u}, t)\right) \dd\vec{u} \dd t \right) \nonumber \\
    &= \exp\left( \int_{\R} \int_{\Sph^{d-1}} \Psi\left(\TOp_{\ReLU}^{\dagger*}\curly{\varphi}(\mat{U} \vec{u}, t)\right) \dd\vec{u} \dd t \right) \nonumber \\
    &= \exp\left( \int_{\R} \int_{\Sph^{d-1}} \Psi\left(\TOp_{\ReLU}^{\dagger*}\curly{\varphi}(\vec{\tilde{u}}, t)\right) \dd\vec{\tilde{u}} \dd t \right) \nonumber \\
    &= \hat{\Pr}_{s_{\ReLU}}(\varphi).
\end{align}    

\item In order to show that $s_{\ReLU}$ (when $\Pr_V$ has zero mean and a finite second moment) is wide-sense self-similar with Hurst exponent $H=3/2$, we will show that for $a > 0$, 
\begin{equation}
    a^{2H} \E\sq{s_{\ReLU}(\vec{x}/a) s_{\ReLU}(\vec{y}/a)} = \E\sq{s_{\ReLU}(\vec{x}) s_{\ReLU}(\vec{y})}.
\end{equation}
Since $\Pr_V$ has zero mean, based on \cref{eq:mean}, we have that $\E\sq{s_{\ReLU}(\vec{x})} = 0$. Thus, using \cref{eq:autocov}, we immediately see that 
\begin{equation}
    \E\sq{s_{\ReLU}(\vec{x}/a) s_{\ReLU}(\vec{y}/a)} = a^{-3} \E\sq{s_{\ReLU}(\vec{x}) s_{\ReLU}(\vec{y})}. \label{eq:dervn_wsss_1}
\end{equation}

\item From the mean and covariance functionals in \cref{eq:dervn_mean_functional,eq:dervn_cov_functional}, respectively, and the form of the characteristic functional \cref{eq:cf_srelu}, we deduce from \cref{defn:Gauss-process} that $s_{\ReLU}$ is \emph{non-Gaussian}, even when $\Pr_V$ has a finite second moment

    \end{enumerate}
\end{proof}

%% file: appendix/asymptotic.tex
\section{Asymptotic Results}
\label{app:asymptotic}

To prove \cref{th:asymptotic}, we rely on a generalized version of the L\'evy continuity theorem from \citet[Theorem 2.3]{bierme2018generalized}, which we state below.

\begin{theorem}[Generalized L\'{e}vy continuity theorem] \label{thm:levy-cont}
    Let $(s_n)_{n \in \N}$ be a sequence of generalized stochastic processes that take values in $\Sch'(\R^d)$ with characteristic functionals $\paren{\hat{\Pr}_{s_n}}_{n \in \N}$. If $\hat{\Pr}_{s_n}$ converges pointwise to a functional $\hat{\PrQ}: \Sch(\R^d) \to \mathbb{C}$ that is continuous at $\vec{0}$, then there exists a generalized stochastic process $s$ such that its characteristic functional satisfies $\hat{\Pr}_{s} = \hat{\PrQ}$ and $s_n \xrightarrow[n \to \infty]{\Law} s$.
\end{theorem}

\begin{proof}[Proof of \cref{th:asymptotic}]
By \cref{thm:levy-cont}, we need to show that
\begin{enumerate}
    \item for every $\varphi \in \Sch(\R^d)$, the sequence $\paren*{\hat{\Pr}_{s_{\ReLU}^n}(\varphi)}_{n \in \N}$ converges to
    \begin{equation}
        \hat{\Pr}_{s_{\ReLU}^\infty}(\varphi) \coloneqq \exp\left( - |b|^{\alpha} \|\TOp_{\ReLU}^{\dagger*}\curly{\varphi} \|_{L^{\alpha}}^{\alpha} \right)
    \end{equation}
    and
    \item the functional $\hat{\Pr}_{s_{\ReLU}^\infty}$ is continuous on $\Sch(\R^d)$.
\end{enumerate}
We first show that, for every $\varphi \in \Sch(\R^d)$,
\begin{equation}\label{eq:conv_cf}
    \lim_{n \rightarrow \infty} \hat{\Pr}_{s_{\ReLU}^n}(\varphi) = \hat{\Pr}_{s_{\ReLU}^\infty}(\varphi).
\end{equation}
Our derivation is inspired from the proof of Lemma~2 of \cite{fageot2020gaussian}. Since $s_{\ReLU}^{n}$ is a \emph{bona fide} generalized stochastic process that takes values in $\Sch'(\R^d)$, the functional $\hat{\Pr}_{s_{\ReLU}^n}(\varphi)$ is well-defined for $\varphi \in \Sch(\R^d)$. On the other hand, we observe that $\hat{\Pr}_{s_{\ReLU}^\infty}(\varphi)$ is also well-defined for $\varphi \in \Sch(\R^d)$ due to \cref{eq:Lp}. Next, we prove the convergence. The characteristic functional of $s_{\ReLU}^n$ is
\begin{equation}
    \hat{\Pr}_{s_{\ReLU}^n}(\varphi) = \exp\left( \int_{\R} \int_{\Sph^{d-1}} \Psi_n\left(\TOp_{\ReLU}^{\dagger*}\curly{\varphi}(\vec{u}, t)\right) \dd\vec{u} \dd t \right),
\end{equation}
where
\begin{equation}
    \Psi_n(\xi) \coloneqq n \left( \e^{-\frac{|b \xi|^{\alpha}}{n}} - 1 \right).
\end{equation}
For a fixed $\vec{z} \in \cyl$, we have that
\begin{equation}
    \Psi_n(\phi(\vec{z})) = n \left( \e^{-\frac{|b \phi(\vec{z})|^{\alpha}}{n}} - 1 \right) \xrightarrow[n \rightarrow \infty]{} -|b \phi(\vec{z}) |^{\alpha},
\end{equation}
where $\phi = \TOp_{\ReLU}^{\dagger*}\curly{\varphi}$. Thus, we need to show that
\begin{equation}\label{eq:conv_integral}
    \int_{\cyl} \Psi_n(\phi(\vec{z})) \dd \vec{z} \xrightarrow[n \rightarrow \infty]{} \int_{\cyl} -|b \phi(\vec{z}) |^{\alpha} \dd \vec{z}.
\end{equation}
From p. 1058 in \cite{fageot2020gaussian}, we have that
\begin{equation}
    |\Psi_n(\phi(\vec{z}))| \leq \sqrt{2} |b \phi(\vec{z}) |^{\alpha}.
\end{equation}
The function $\vec z \mapsto |b \phi(\vec{z}) |^{\alpha}$ is in $L^1(\cyl)$ due to \cref{eq:Lp}. Thus, we can apply the Lebesgue dominated convergence theorem to show that \cref{eq:conv_integral}, and consequently \cref{eq:conv_cf}, holds. Finally, the continuity of $\hat{\Pr}_{s_{\ReLU}^\infty}(\varphi)$ on $\Sch(\R^d)$ follows from the fact that the operator $\TOp_{\ReLU}^{\dagger*}$ continuously maps $\Sch(\R^d)$ to $L^p(\cyl)$ for $p \in [1,2]$ (cf., \Cref{eq:Lp}).
\end{proof}

\paragraph{Gaussianity of $s_{\ReLU}^{\infty}$} When $\alpha=2$, the characteristic functional of $s_{\ReLU}^{\infty}$ can be written as
\begin{equation} \label{eq:dervn_widelimit_gaussian_cf}
    \hat{\Pr}_{s_{\ReLU}^{\infty}}(\varphi) = \exp\left( \int_{\R} \int_{\Sph^{d-1}} \Psi_{\infty}\left(\TOp_{\ReLU}^{\dagger*}\curly{\varphi}(\vec{u}, t)\right) \dd\vec{u} \dd t \right),
\end{equation}
where $\varphi \in \Sch(\R^d)$ and $\Psi_{\infty}(\xi) = -|b \xi|^{2}$ for $\xi \in \R$. Using the moment generating properties of the characteristic functional (as in \cref{app:properties}), we get that the mean functional is 
\begin{equation}\label{eq:dervn_widelimit_gaussian_mean_functional}
    \mu_{s_{\ReLU}^{\infty}}(\varphi) = 0, \quad \varphi \in \Sch(\R^d),
\end{equation}
as $\Psi_{\infty}'(0) = 0$, and the covariance functional is
\begin{equation}\label{eq:dervn_widelimit_gaussian_cov_functional}
    \Sigma_{s_{\ReLU}^{\infty}}(\varphi_1, \varphi_2) = 2|b|^{2} \int_{\cyl} \TOp_{\ReLU}^{\dagger*}\curly{\varphi_1}(\vec{z}) \TOp_{\ReLU}^{\dagger*}\curly{\varphi_2}(\vec{z}) \dd\vec{z}, \quad \varphi_1, \varphi_2 \in \Sch(\R^d),
\end{equation}
as $\Psi_{\infty}''(0) = -2|b|^{2}$. Thus, from \cref{eq:dervn_widelimit_gaussian_cf,eq:dervn_widelimit_gaussian_mean_functional,eq:dervn_widelimit_gaussian_cov_functional,defn:Gauss-process}, we see that $s_{\ReLU}^{\infty}$ is a Gaussian process when $\alpha=2$.

%% file: appendix/more-figs.tex
\section{Discussion of the Numerical Examples} \label{app:more-figs}

We generated realizations of the random neural networks by taking advantage of the property that Poisson points are uniformly distributed in each finite volume (cf., \Cref{eq:Poi-unif}) combined with the fact that the width of a random neural network observed on a compact domain is a Poisson random variable with mean proportional to the rate parameter $\lambda$ multiplied by a property related to the geometry of the domain (cf., \cref{subsec:bounded-domain}).
In particular, the random neural network realizations in \cref{fig:Gauss-2D,fig:stable-2D} were plotted on the compact domain $\Omega = [-1, +1]^d$ and were generated according to the following procedure.
\begin{enumerate}
    \item Generate a Poisson random variable $N_{\lambda, \Omega}$ with mean $\lambda \abs{\mathcal{Z}_{\Omega}}$, where $\mathcal{Z}_{\Omega}$ was defined in \cref{eq:Z-Omega}.

    \item Generate $N_{\lambda, \Omega}$ points i.i.d.\ uniformly on the finite volume $\mathcal{Z}_{\Omega} \subset \cyl$, which we denote by $\curly{(\vec{w}_k, b_k)}_{k=1}^{N_{\lambda, \Omega}}$.

    \item Generate $N_{\lambda, \Omega}$ i.i.d.\ random variables according to the law $\Pr_V$, which we denote by $\curly{v_k}_{k=1}^{N_{\lambda, \Omega}}$.

    \item Construct the random neural network
    \begin{equation}
        s_{\ReLU}^\text{numeric}(\vec x) = \sum_{k=1}^{N_{\lambda, \Omega}} v_k \sq*{\ReLU(\vec{w}_k^\T\vec x - b_k) + \vec{c}_k^{\T}\vec x + c_{0, k}}
    \end{equation}
    according to the computation in \cref{eq:Poi-to-ReLU} with $\varepsilon \to 0$.
\end{enumerate}

The resulting random neural network $s_{\ReLU}^\text{numeric}$ is, up to an affine function, a realization of $\ReLUP(\lambda; \Pr_V)$. Finally, in order to highlight the linear regions of the generated networks, we color the top-down plots in \cref{fig:Gauss-2D,fig:stable-2D} according to the magnitude of the gradient of $s_{\ReLU}^\text{numeric}$. As the color map choice is arbitrary, the resulting plots are thus realizations of $\ReLUP(\lambda; \Pr_V)$ (since the magnitude of the gradient of an affine function is a constant, and therefore simply shifts the color map). We include some additional plots of the random neural networks in \cref{fig:Gauss-3D,fig:stable-3D}. These figures are surface plots of the random neural networks in \cref{fig:Gauss-2D,fig:stable-2D}, respectively.

\begin{figure}
    \centering
    \begin{minipage}[b]{0.24\linewidth}
        \centering
        \centerline{\includegraphics[width=\textwidth]{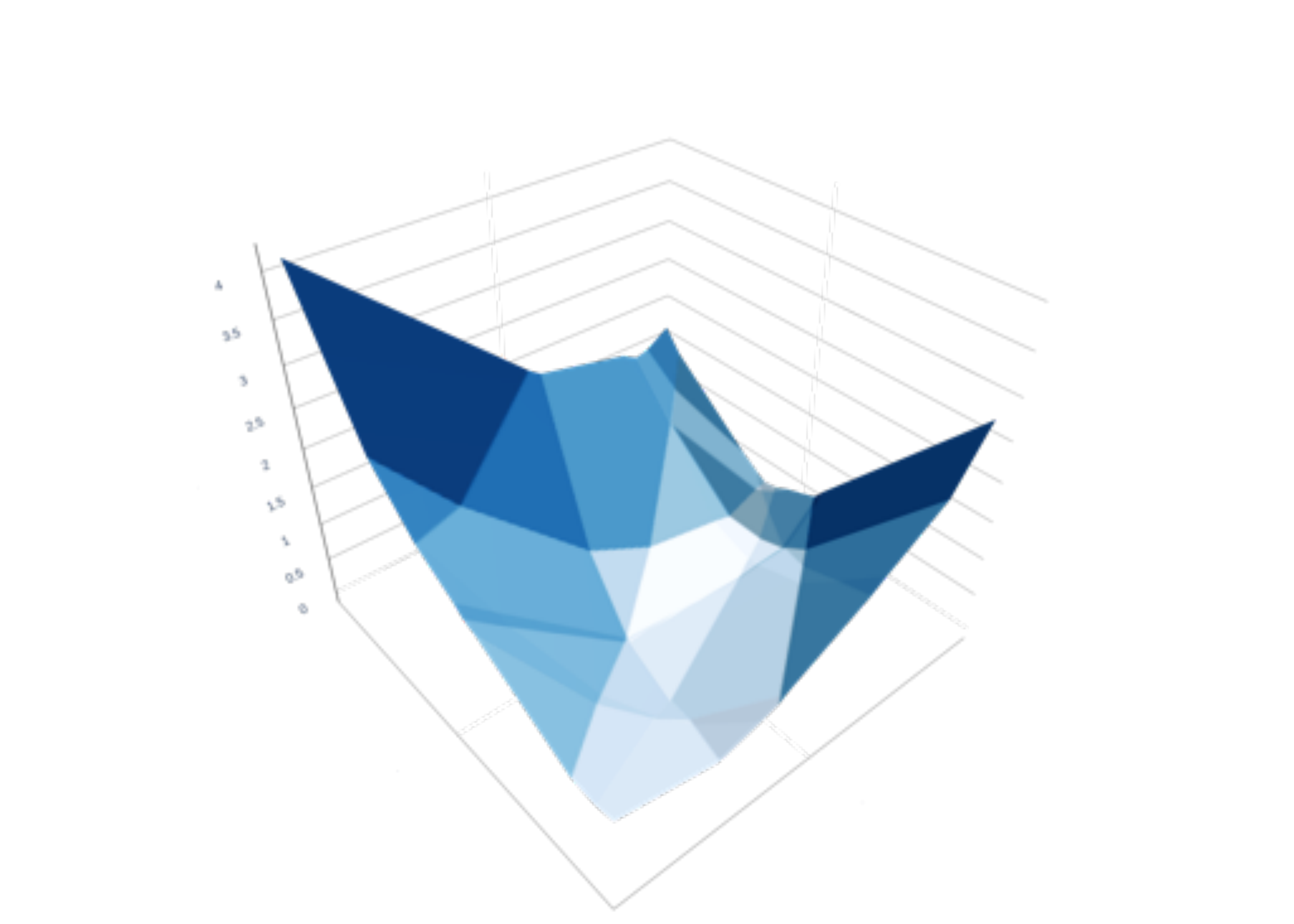}}
        (a) $\lambda = 1$
    \end{minipage}
    \begin{minipage}[b]{0.24\linewidth}
        \centering
        \centerline{\includegraphics[width=\textwidth]{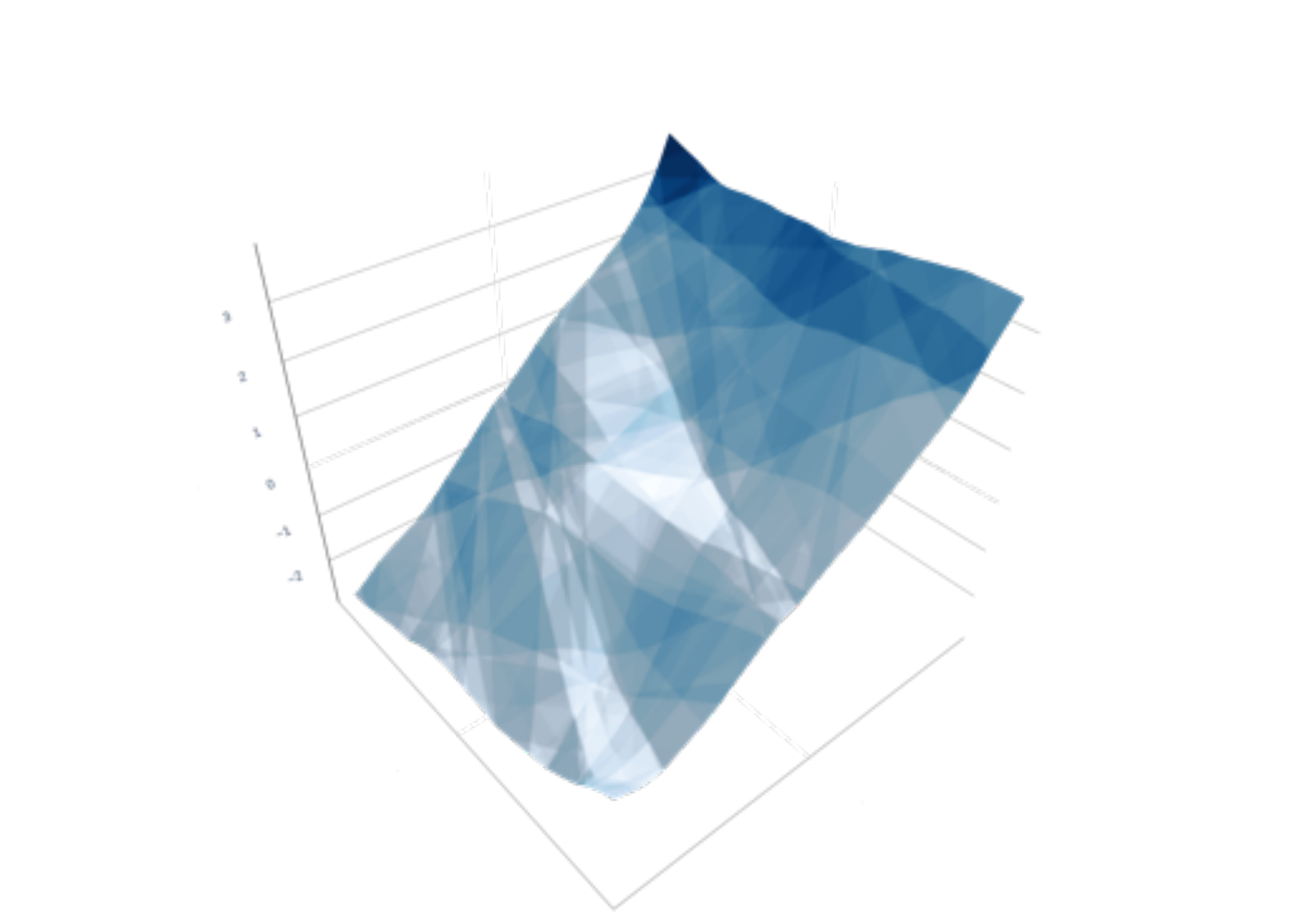}}
        (b) $\lambda = 10$
    \end{minipage}
    \begin{minipage}[b]{0.24\linewidth}
        \centering
        \centerline{\includegraphics[width=\textwidth]{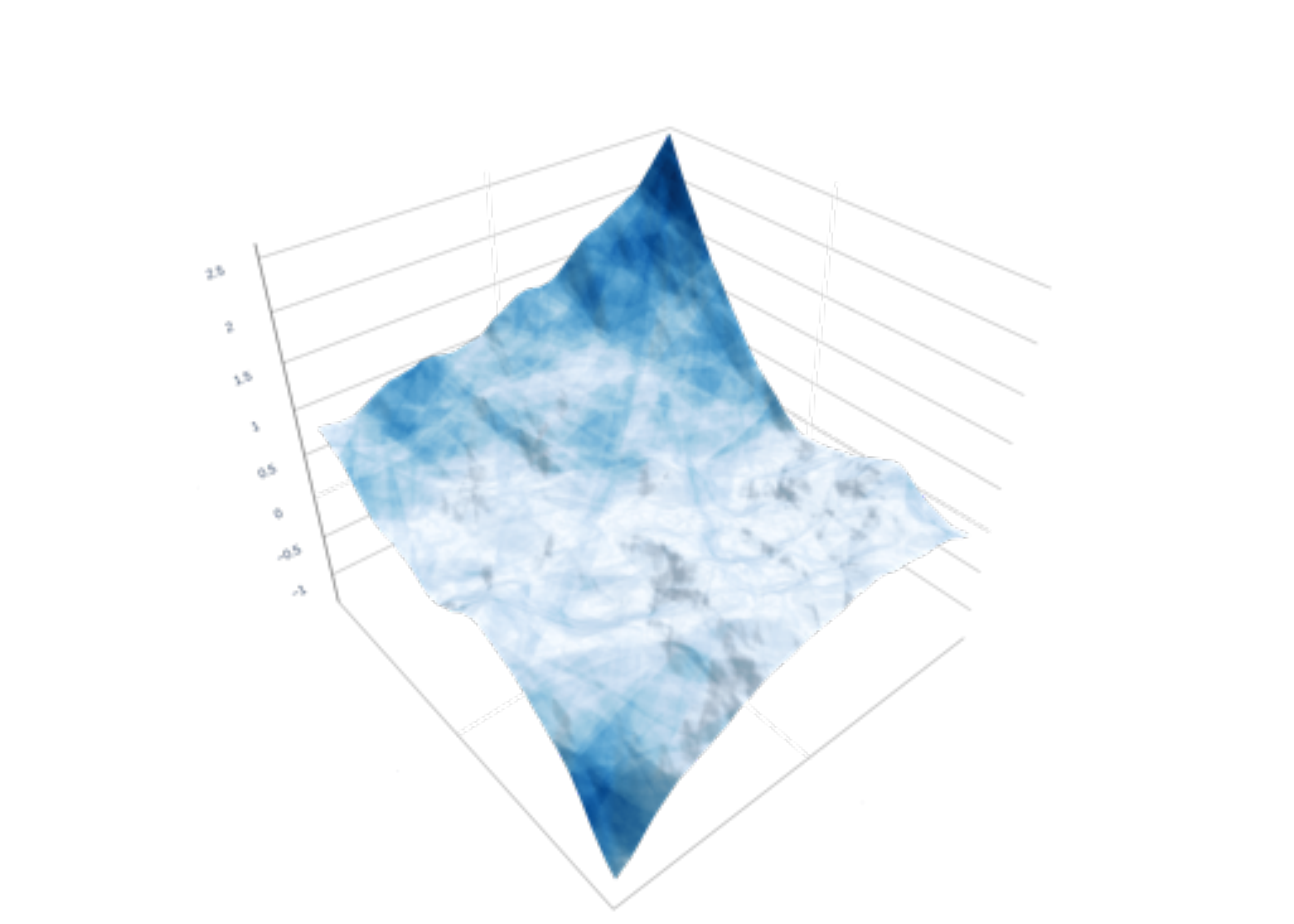}}
        (c) $\lambda = 100$
    \end{minipage}
    \begin{minipage}[b]{0.24\linewidth}
        \centering
        \centerline{\includegraphics[width=\textwidth]{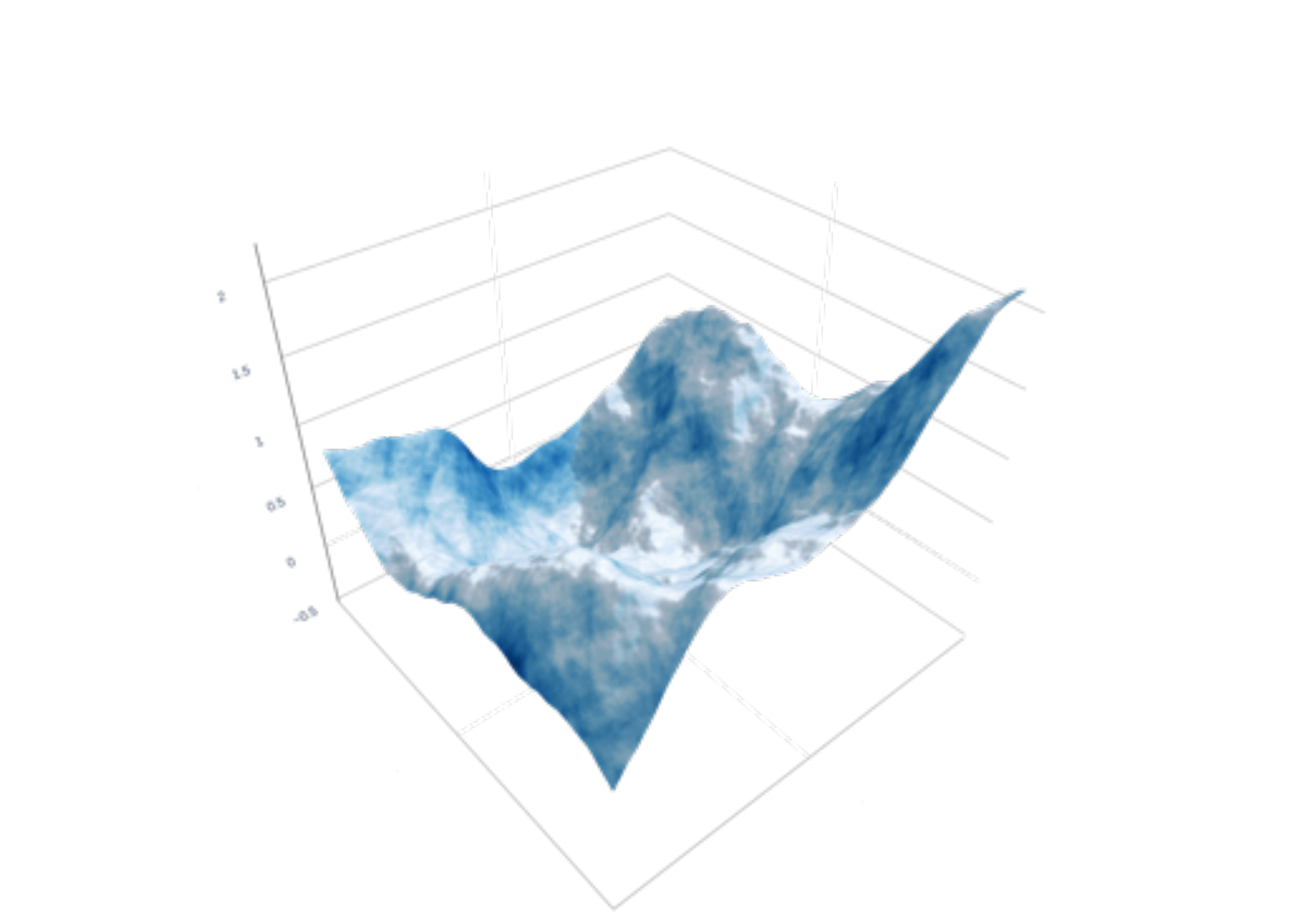}}
        (d) $\lambda = 1000$
    \end{minipage}
    \caption{$\Pr_V$ is Gaussian.}
    \label{fig:Gauss-3D}
\end{figure}

\begin{figure}
    \centering
    \begin{minipage}[b]{0.24\linewidth}
        \centering
        \centerline{\includegraphics[width=\textwidth]{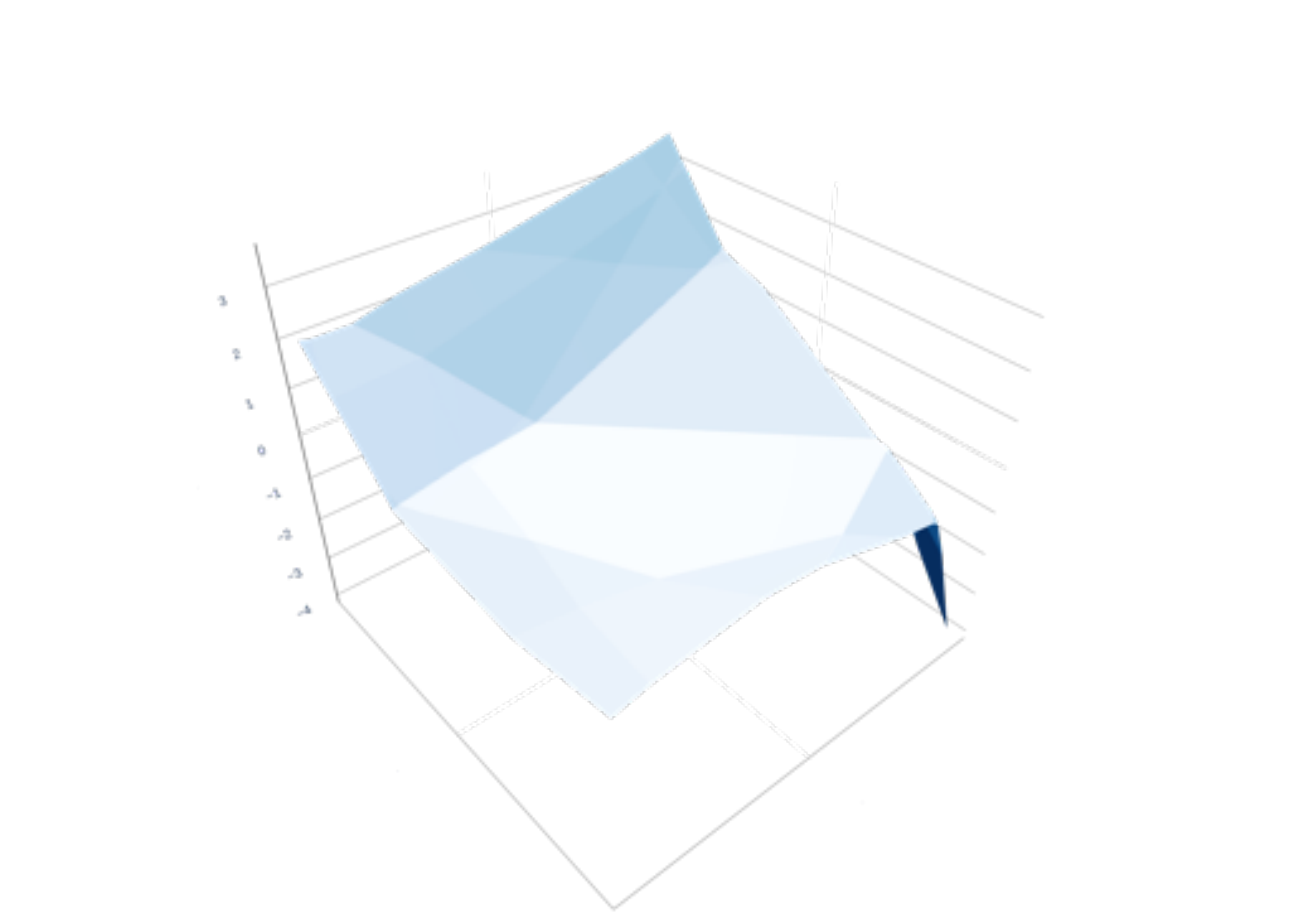}}
        (a) $\lambda = 1$
    \end{minipage}
    \begin{minipage}[b]{0.24\linewidth}
        \centering
        \centerline{\includegraphics[width=\textwidth]{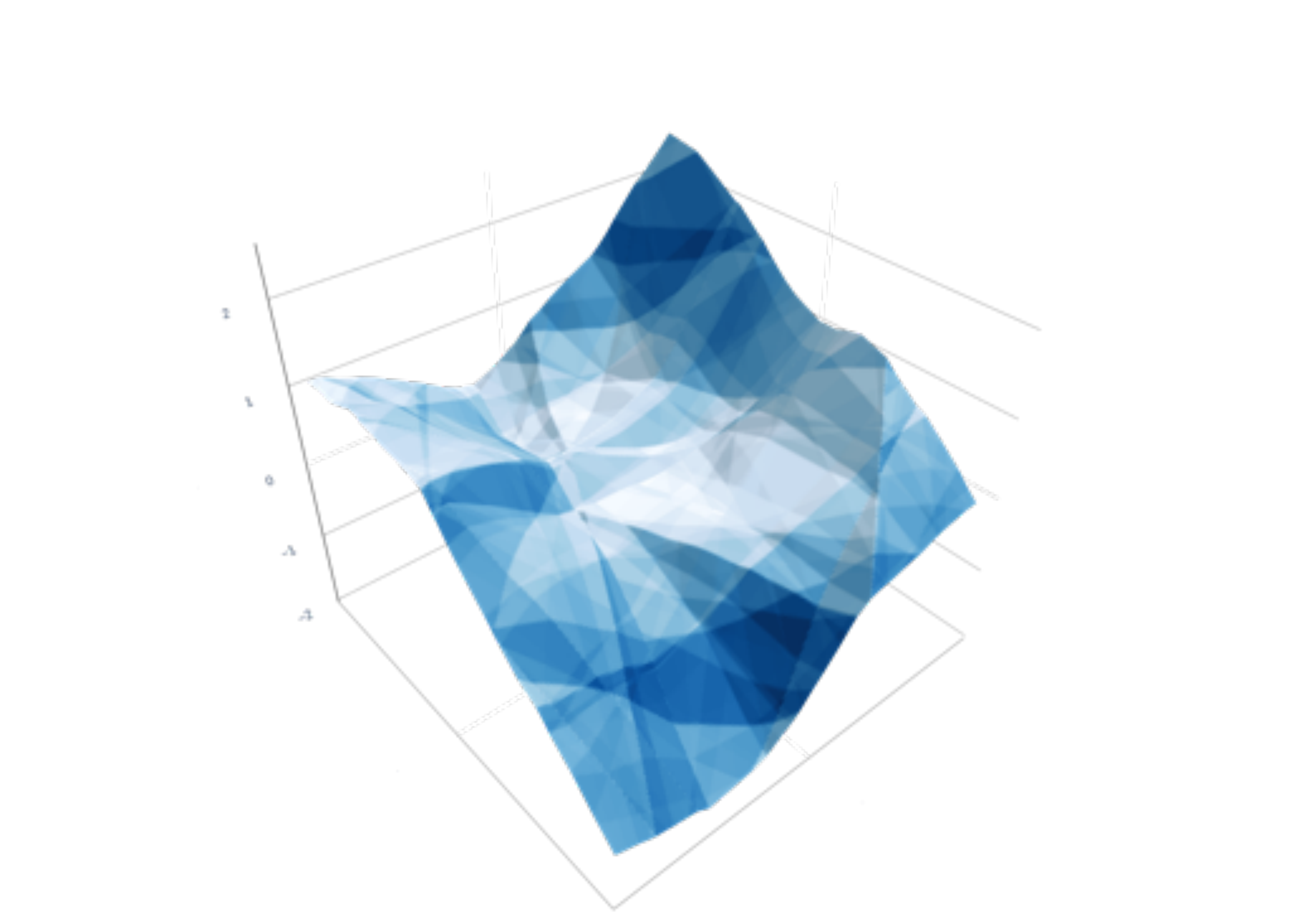}}
        (b) $\lambda = 10$
    \end{minipage}
    \begin{minipage}[b]{0.24\linewidth}
        \centering
        \centerline{\includegraphics[width=\textwidth]{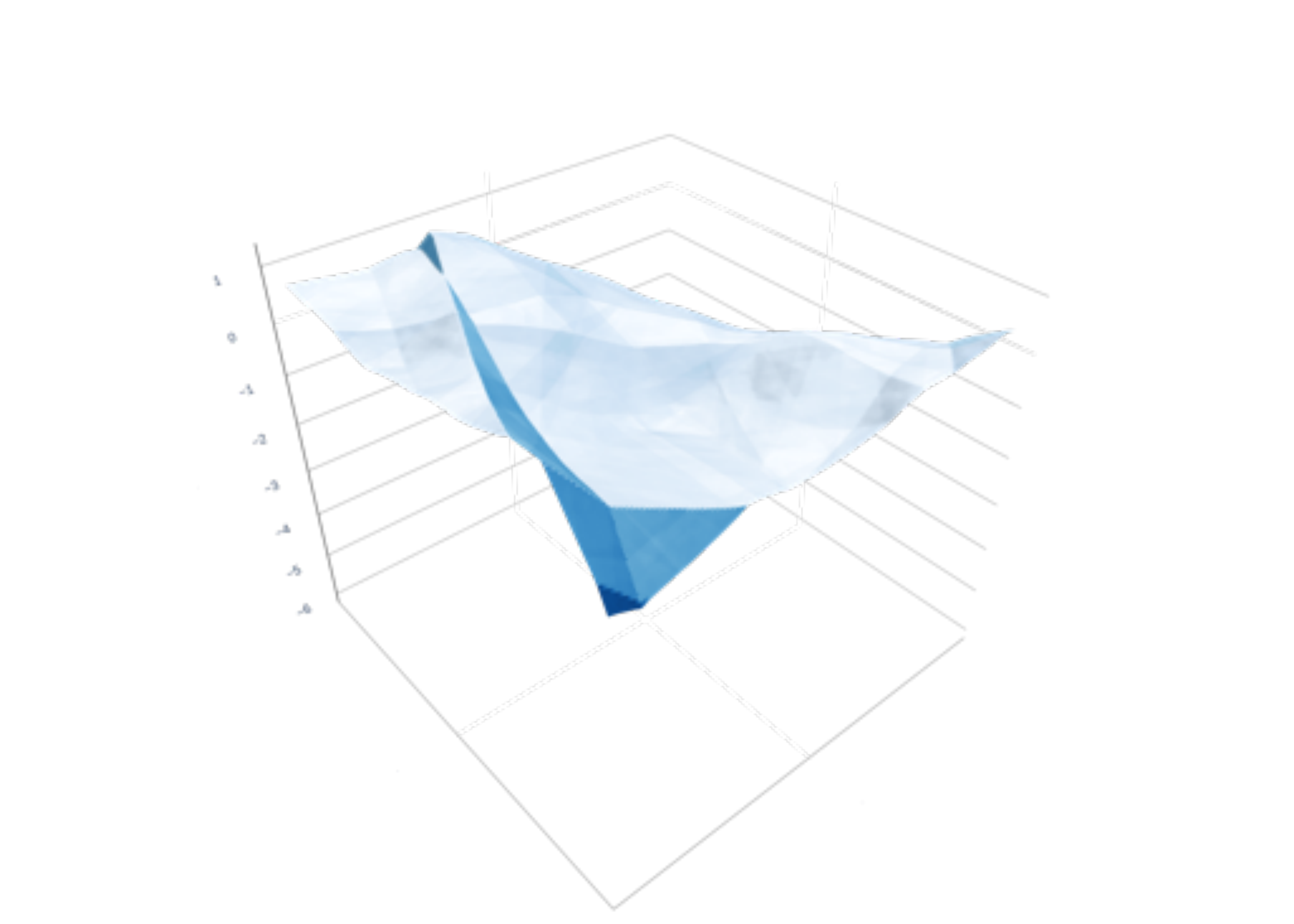}}
        (c) $\lambda = 100$
    \end{minipage}
    \begin{minipage}[b]{0.24\linewidth}
        \centering
        \centerline{\includegraphics[width=\textwidth]{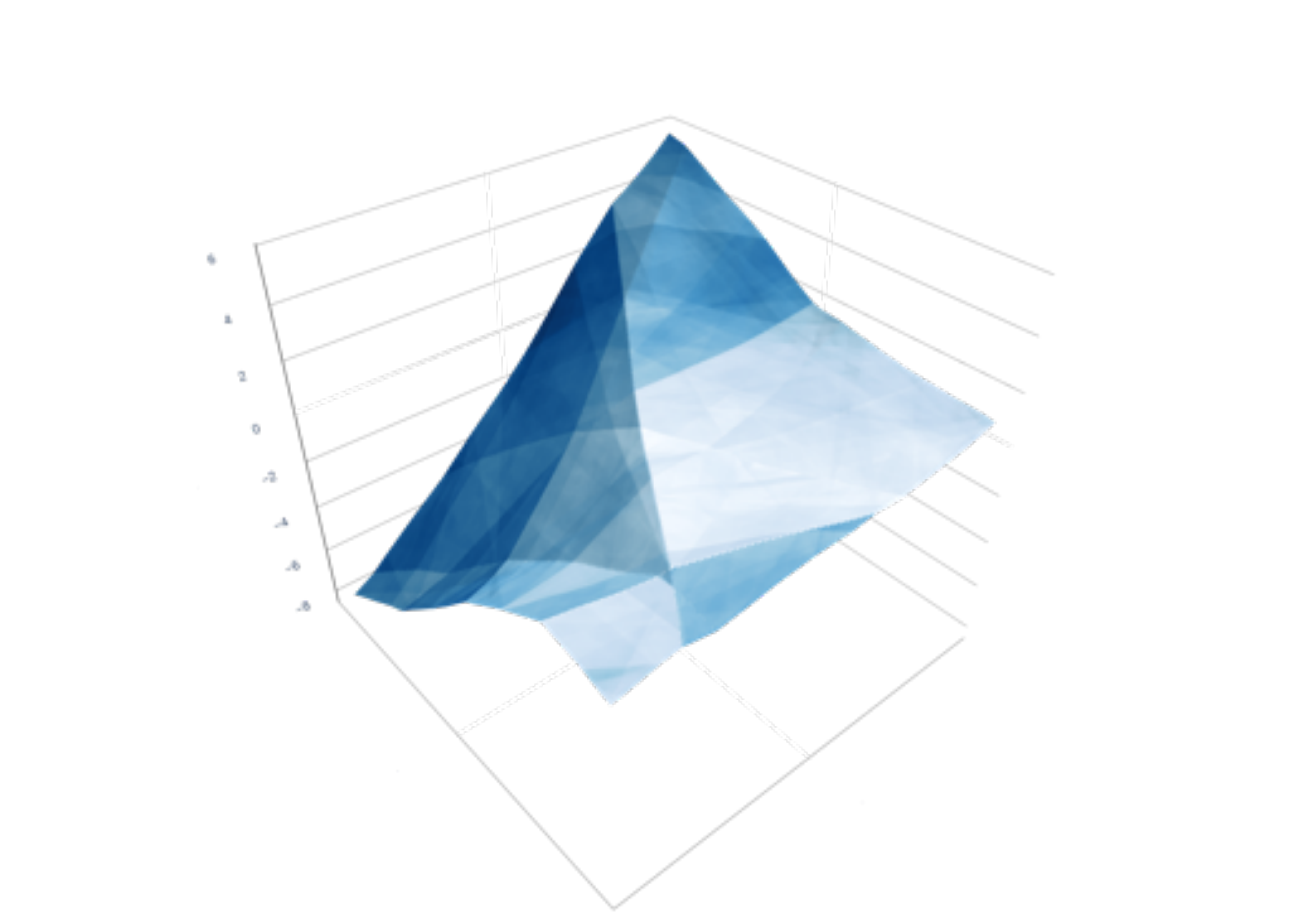}}
        (d) $\lambda = 1000$
    \end{minipage}
    \caption{$\Pr_V$ is symmetric ($\alpha = 1.25$)-stable.}
    \label{fig:stable-3D}
\end{figure}

%% file: main.bbl
\begin{thebibliography}{46}
\providecommand{\natexlab}[1]{#1}
\providecommand{\url}[1]{\texttt{#1}}
\expandafter\ifx\csname urlstyle\endcsname\relax
  \providecommand{\doi}[1]{doi: #1}\else
  \providecommand{\doi}{doi: \begingroup \urlstyle{rm}\Url}\fi

\bibitem[Bartolucci et~al.(2023)Bartolucci, De~Vito, Rosasco, and
  Vigogna]{bartolucci2023understanding}
Francesca Bartolucci, Ernesto De~Vito, Lorenzo Rosasco, and Stefano Vigogna.
\newblock Understanding neural networks with reproducing kernel {B}anach
  spaces.
\newblock \emph{Applied and Computational Harmonic Analysis}, 62:\penalty0
  194--236, 2023.

\bibitem[Bierm{\'e} et~al.(2018)Bierm{\'e}, Durieu, and
  Wang]{bierme2018generalized}
Hermine Bierm{\'e}, Olivier Durieu, and Yizao Wang.
\newblock Generalized random fields and {L}{\'e}vy's continuity theorem on the
  space of tempered distributions.
\newblock \emph{Communications on Stochastic Analysis}, 12\penalty0
  (4):\penalty0 4, 2018.

\bibitem[Daley and Vere-Jones(2007)]{daley2007introduction}
Daryl~J. Daley and David Vere-Jones.
\newblock \emph{An Introduction to the Theory of Point Processes: Volume II:
  General Theory and Structure}.
\newblock Probability and Its Applications. Springer New York, 2007.

\bibitem[Duttweiler and Kailath(1973)]{duttweiler1973rkhs}
Donald~L. Duttweiler and Thomas Kailath.
\newblock {RKHS} approach to detection and estimation problems--{IV}:
  Non-{G}aussian detection.
\newblock \emph{IEEE Transactions on Information Theory}, 19\penalty0
  (1):\penalty0 19--28, 1973.

\bibitem[Dyer and Gur-Ari(2020)]{dyer2019asymptotics}
Ethan Dyer and Guy Gur-Ari.
\newblock Asymptotics of wide networks from {F}eynman diagrams.
\newblock In \emph{International Conference on Learning Representations}, 2020.

\bibitem[Fageot and Unser(2019)]{fageot2019scaling}
Julien Fageot and Michael Unser.
\newblock Scaling limits of solutions of linear stochastic differential
  equations driven by {L}{\'e}vy white noises.
\newblock \emph{Journal of Theoretical Probability}, 32\penalty0 (3):\penalty0
  1166--1189, 2019.

\bibitem[Fageot et~al.(2014)Fageot, Amini, and Unser]{fageot2014continuity}
Julien Fageot, Arash Amini, and Michael Unser.
\newblock On the continuity of characteristic functionals and sparse stochastic
  modeling.
\newblock \emph{Journal of Fourier Analysis and Applications}, 20:\penalty0
  1179--1211, 2014.

\bibitem[Fageot et~al.(2020)Fageot, Uhlmann, and Unser]{fageot2020gaussian}
Julien Fageot, Virginie Uhlmann, and Michael Unser.
\newblock Gaussian and sparse processes are limits of generalized {P}oisson
  processes.
\newblock \emph{Applied and Computational Harmonic Analysis}, 48\penalty0
  (3):\penalty0 1045--1065, 2020.

\bibitem[Fernique(1967)]{fernique1967processus}
Xavier Fernique.
\newblock Processus lin{\'e}aires, processus g{\'e}n{\'e}ralis{\'e}s.
\newblock \emph{Annales de l'institut Fourier}, 17\penalty0 (1):\penalty0
  1--92, 1967.

\bibitem[Garriga-Alonso et~al.(2019)Garriga-Alonso, Rasmussen, and
  Aitchison]{garriga-alonso2018deep}
Adrià Garriga-Alonso, Carl~Edward Rasmussen, and Laurence Aitchison.
\newblock Deep convolutional networks as shallow {G}aussian processes.
\newblock In \emph{International Conference on Learning Representations}, 2019.

\bibitem[Gelfand(1955)]{MR0068769}
Izrail~M. Gelfand.
\newblock Generalized random processes.
\newblock \emph{Dokl. Akad. Nauk SSSR (N.S.)}, 100:\penalty0 853--856, 1955.

\bibitem[Gelfand and Shilov(1964)]{GelfandV1}
Izrail~M. Gelfand and Georgiy~E. Shilov.
\newblock \emph{Generalized functions. {V}ol. {I}: {P}roperties and
  operations}.
\newblock Academic Press, 1964.

\bibitem[Gelfand and Vilenkin(1964)]{gelfand1964generalized}
Izrail~M. Gelfand and Naum~Ya. Vilenkin.
\newblock \emph{Generalized functions, {V}ol. 4: Applications of harmonic
  analysis}.
\newblock Academic Press, 1964.

\bibitem[Gelfand et~al.(1966)Gelfand, Graev, and
  Vilenkin]{GelfandIntegralGeometry}
Izrail~M. Gelfand, Mark~I. Graev, and Naum~Ya. Vilenkin.
\newblock \emph{Generalized functions. {V}ol. 5: {I}ntegral geometry and
  representation theory}.
\newblock Academic Press, 1966.

\bibitem[Hanin(2023)]{hanin2021random}
Boris Hanin.
\newblock Random neural networks in the infinite width limit as {G}aussian
  processes.
\newblock \emph{The Annals of Applied Probability}, 33\penalty0 (6A):\penalty0
  4798--4819, 2023.

\bibitem[Helgason(2011)]{HelgasonIntegralGeometry}
Sigurdur Helgason.
\newblock \emph{Integral Geometry and {R}adon Transforms}.
\newblock Springer New York, 2011.

\bibitem[Hida and Ikeda(1967)]{hida1967analysis}
Takeyuki Hida and Nobuyuki Ikeda.
\newblock Analysis on {H}ilbert space with reproducing kernel arising from
  multiple {W}iener integral.
\newblock In \emph{Proc. {F}ifth {B}erkeley {S}ympos. {M}ath. {S}tatist. and
  {P}robability}, pages 117--143. Univ. California Press, Berkeley, CA, 1967.

\bibitem[It\^{o}(1954)]{MR0065060}
Kiyosi It\^{o}.
\newblock Stationary random distributions.
\newblock \emph{Memoirs of the College of Science. University of Kyoto. Series
  A. Mathematics}, 28:\penalty0 209--223, 1954.

\bibitem[Itô(1984)]{ito1984foundations}
Kiyosi Itô.
\newblock \emph{Foundations of stochastic differential equations in infinite
  dimensional spaces}, volume~47.
\newblock SIAM, 1984.

\bibitem[Jacob and Schilling(2001)]{Jacob2001}
Niels Jacob and Ren{\'e}~L. Schilling.
\newblock \emph{L{\'e}vy-Type Processes and Pseudodifferential Operators},
  pages 139--168.
\newblock Birkh{\"a}user Boston, Boston, MA, 2001.
\newblock ISBN 978-1-4612-0197-7.

\bibitem[Kolmogorov(1935)]{kolmogorov1935transformation}
Andrei~N. Kolmogorov.
\newblock La transformation de {L}aplace dans les espaces lin{\'e}aires.
\newblock \emph{CR Acad. Sci. Paris}, 200:\penalty0 1717--1718, 1935.

\bibitem[Lee et~al.(2018)Lee, Bahri, Novak, Schoenholz, Pennington, and
  Sohl-Dickstein]{lee2018deep}
Jaehoon Lee, Yasaman Bahri, Roman Novak, Samuel~S. Schoenholz, Jeffrey
  Pennington, and Jascha Sohl-Dickstein.
\newblock Deep neural networks as {G}aussian processes.
\newblock In \emph{International Conference on Learning Representations}, 2018.

\bibitem[Ludwig(1966)]{LudwigRadon}
Donald Ludwig.
\newblock The {R}adon transform on {E}uclidean space.
\newblock \emph{Communications on Pure and Applied Mathematics}, 19:\penalty0
  49--81, 1966.

\bibitem[Mandelbrot and Van~Ness(1968)]{mandelbrot1968fractional}
Benoit~B. Mandelbrot and John~W. Van~Ness.
\newblock Fractional {B}rownian motions, fractional noises and applications.
\newblock \emph{SIAM Review}, 10\penalty0 (4):\penalty0 422--437, 1968.

\bibitem[Matthews et~al.(2018)Matthews, Hron, Rowland, Turner, and
  Ghahramani]{matthews2018gaussian}
Alexander G. de~G. Matthews, Jiri Hron, Mark Rowland, Richard~E. Turner, and
  Zoubin Ghahramani.
\newblock Gaussian process behaviour in wide deep neural networks.
\newblock In \emph{International Conference on Learning Representations}, 2018.

\bibitem[Minlos(1959)]{minlos1959generalized}
Robert~A. Minlos.
\newblock Generalized random processes and their extension in measure.
\newblock \emph{Trudy Moskovskogo Matematicheskogo Obshchestva}, 8:\penalty0
  497--518, 1959.

\bibitem[Neal(1996)]{neal1996bayesian}
Radford~M. Neal.
\newblock \emph{Bayesian Learning for Neural Networks}.
\newblock Lecture Notes in Statistics. Springer New York, 1996.

\bibitem[Novak et~al.(2019)Novak, Xiao, Bahri, Lee, Yang, Abolafia, Pennington,
  and Sohl-Dickstein]{novak2019bayesian}
Roman Novak, Lechao Xiao, Yasaman Bahri, Jaehoon Lee, Greg Yang, Daniel~A.
  Abolafia, Jeffrey Pennington, and Jascha Sohl-Dickstein.
\newblock Bayesian deep convolutional networks with many channels are
  {G}aussian processes.
\newblock In \emph{International Conference on Learning Representations}, 2019.

\bibitem[Ongie et~al.(2020)Ongie, Willett, Soudry, and
  Srebro]{ongie2019function}
Greg Ongie, Rebecca Willett, Daniel Soudry, and Nathan Srebro.
\newblock A function space view of bounded norm infinite width {ReLU} nets: The
  multivariate case.
\newblock In \emph{International Conference on Learning Representations}, 2020.

\bibitem[Parhi and Nowak(2021)]{parhi2021banach}
Rahul Parhi and Robert~D. Nowak.
\newblock Banach space representer theorems for neural networks and ridge
  splines.
\newblock \emph{Journal of Machine Learning Research}, 22\penalty0
  (43):\penalty0 1--40, 2021.

\bibitem[Parhi and Nowak(2022)]{parhi2022kinds}
Rahul Parhi and Robert~D. Nowak.
\newblock What kinds of functions do deep neural networks learn? {I}nsights
  from variational spline theory.
\newblock \emph{SIAM Journal on Mathematics of Data Science}, 4\penalty0
  (2):\penalty0 464--489, 2022.

\bibitem[Parhi and Nowak(2023{\natexlab{a}})]{parhi2022near}
Rahul Parhi and Robert~D. Nowak.
\newblock Near-minimax optimal estimation with shallow {ReLU} neural networks.
\newblock \emph{IEEE Transactions on Information Theory}, 69\penalty0
  (2):\penalty0 1125--1140, 2023{\natexlab{a}}.

\bibitem[Parhi and Nowak(2023{\natexlab{b}})]{parhi2023deep}
Rahul Parhi and Robert~D. Nowak.
\newblock Deep learning meets sparse regularization: A signal processing
  perspective.
\newblock \emph{IEEE Signal Processing Magazine}, 40\penalty0 (6):\penalty0
  63--74, 2023{\natexlab{b}}.

\bibitem[Parhi and Unser(2024)]{parhi2024distributional}
Rahul Parhi and Michael Unser.
\newblock Distributional extension and invertibility of the {$k$}-plane
  transform and its dual.
\newblock \emph{SIAM Journal on Mathematical Analysis}, 56\penalty0
  (4):\penalty0 4662--4686, 2024.

\bibitem[Parhi and Unser(2025)]{parhi2025function}
Rahul Parhi and Michael Unser.
\newblock Function-space optimality of neural architectures with multivariate
  nonlinearities.
\newblock \emph{SIAM Journal on Mathematics of Data Science}, 7\penalty0
  (1):\penalty0 110--135, 2025.

\bibitem[Ramm and Katsevich(1996)]{RammRadonBook}
Alexander~G. Ramm and Alexander~I. Katsevich.
\newblock \emph{The {R}adon transform and local tomography}.
\newblock CRC Press, Boca Raton, FL, 1996.

\bibitem[Rudin(1991)]{RudinFA}
Walter Rudin.
\newblock \emph{Functional analysis}.
\newblock International Series in Pure and Applied Mathematics. McGraw-Hill,
  Inc., New York, second edition, 1991.

\bibitem[Sato(1999)]{ken1999levy}
Ken-Iti Sato.
\newblock \emph{L{\'e}vy Processes and Infinitely Divisible Distributions}.
\newblock Cambridge Studies in Advanced Mathematics. Cambridge University
  Press, 1999.

\bibitem[Shenouda et~al.(2024)Shenouda, Parhi, Lee, and
  Nowak]{shenouda2024variation}
Joseph Shenouda, Rahul Parhi, Kangwook Lee, and Robert~D. Nowak.
\newblock Variation spaces for multi-output neural networks: Insights on
  multi-task learning and network compression.
\newblock \emph{Journal of Machine Learning Research}, 25\penalty0
  (231):\penalty0 1--40, 2024.

\bibitem[Unser(2023)]{UnserRidges}
Michael Unser.
\newblock Ridges, neural networks, and the {R}adon transform.
\newblock \emph{Journal of Machine Learning Research}, 24\penalty0
  (37):\penalty0 1--33, 2023.

\bibitem[Unser and Tafti(2014)]{unser2014introduction}
Michael Unser and Pouya~D. Tafti.
\newblock \emph{An introduction to sparse stochastic processes}.
\newblock Cambridge University Press, 2014.

\bibitem[Unser et~al.(2014)Unser, Tafti, and Sun]{unser2014unified}
Michael Unser, Pouya~D. Tafti, and Qiyu Sun.
\newblock A unified formulation of {G}aussian versus sparse stochastic
  processes—{P}art {I}: Continuous-domain theory.
\newblock \emph{IEEE Transactions on Information Theory}, 60\penalty0
  (3):\penalty0 1945--1962, 2014.

\bibitem[Williams(1996)]{williams1996computing}
Christopher Williams.
\newblock Computing with infinite networks.
\newblock \emph{Advances in Neural Information Processing Systems}, 9, 1996.

\bibitem[Yaida(2020)]{yaida2020non}
Sho Yaida.
\newblock Non-{G}aussian processes and neural networks at finite widths.
\newblock In \emph{Mathematical and Scientific Machine Learning}, pages
  165--192. PMLR, 2020.

\bibitem[Yang(2019)]{yang2019wide}
Greg Yang.
\newblock Tensor programs {I}: Wide feedforward or recurrent neural networks of
  any architecture are {G}aussian processes.
\newblock \emph{Advances in Neural Information Processing Systems}, 32, 2019.

\bibitem[Zavatone-Veth and Pehlevan(2021)]{zavatone2021exact}
Jacob Zavatone-Veth and Cengiz Pehlevan.
\newblock Exact marginal prior distributions of finite {B}ayesian neural
  networks.
\newblock \emph{Advances in Neural Information Processing Systems},
  34:\penalty0 3364--3375, 2021.

\end{thebibliography}
